\documentclass[a4paper,UKenglish,cleveref, autoref, thm-restate]{oasics-v2021}

\pdfoutput=1 
\hideOASIcs 

\usepackage{booktabs}
\usepackage{adjustbox}
\usepackage{xspace}
\usepackage{mathtools}

\definecolor{darkgreen}{RGB}{34,139,34} 
\usepackage{longtable} 
\usepackage{pdflscape} 

\usepackage{booktabs}

\usepackage[ruled,linesnumbered,algosection,noend]{algorithm2e} 
\SetAlgoSkip{4pt}

\usepackage{ragged2e}
\usepackage{tcolorbox}
\tcbuselibrary{skins}
\usepackage{array}
\definecolor{ovgray}{gray}{0.97}
\newcolumntype{Y}{>{\RaggedRight\arraybackslash}X}
\newcolumntype{K}{>{\bfseries\RaggedRight\arraybackslash}p{0.20\linewidth}}

\newcolumntype{L}{>{\RaggedRight\arraybackslash}p{0.22\linewidth}}
\newcolumntype{R}{>{\RaggedRight\arraybackslash}p{0.74\linewidth}}

\newcommand{\ATFCMOPT}{\ensuremath{\texttt{ATFCM}^{OPT}}\xspace}

\newcommand{\ATFCM}{\ensuremath{\texttt{ATFCM}}\xspace}

\newcommand{\ATFCMoDELAY}{\ensuremath{K_d}\xspace}
\newcommand{\ATFCMoNUMSEC}{\ensuremath{K_{\#s}}\xspace}
\newcommand{\ATFCMoSECCHANGES}{\ensuremath{K_{ds}}\xspace}
\newcommand{\ATFCMoREROUTED}{\ensuremath{K_{r}}\xspace}
\newcommand{\ATFCMoRECONFIG}{\ensuremath{K_{c}}\xspace}

\makeatletter
\providecommand{\leftsquigarrow}{\mathrel{\mathpalette\reflect@squig\relax}}
\newcommand{\reflect@squig}[2]{\reflectbox{$\m@th#1\rightsquigarrow$}}
\makeatother
\definecolor{backcolor}{rgb}{0.95,0.95,0.95}

\lstdefinestyle{mystyle}{
    backgroundcolor=\color{backcolor},
    breaklines=true,
    numbers=left,
    numbersep=5pt,
    basicstyle=\ttfamily\scriptsize,
    captionpos=b,
    literate={:-}{{$\;\leftarrow$}}1
             {!=}{{$\neq$}}1
             {:~}{{$\;\leftsquigarrow$}}1,
    mathescape=true
}

\title{ASPaeroFlow: Decomposition Heuristics for Joint Air Traffic Flow \& Capacity Management}
\titlerunning{Decomposition Heuristics for Joint Air Traffic Flow \& Capacity Management}
\author{Alexander Beiser}{TU Wien, Vienna, Austria}{alexander.beiser@tuwien.ac.at}{https://orcid.org/0009-0009-4252-1043}{}
\author{Markus Hecher}{University of Potsdam, Potsdam, Germany \& CNRS, Artois University (CRIL), Artois, France}{hecher@cril.fr}{https://orcid.org/0000-0003-0131-6771}{}
\author{Nysret Musliu}{TU Wien, Vienna, Austria}{nysret.musliu@tuwien.ac.at}{https://orcid.org/0000-0002-3992-8637}{}
\author{Georg Trausmuth}{Frequentis AG, Vienna, Austria}{Georg.TRAUSMUTH@frequentis.com}{}{}
\author{Stefan Woltran}{TU Wien, Vienna, Austria}{stefan.woltran@tuwien.ac.at}{https://orcid.org/0000-0003-1594-8972}{}

\authorrunning{Beiser et al.} 

\Copyright{Alexander Beiser, Markus Hecher, Nysret Musliu, Georg Trausmuth, and Stefan Woltran} 


\category{} 

\relatedversion{The full version excluding supplementary material is published at ATMOS 2026:\\ \url{https://doi.org/10.4230/OASIcs.ATMOS.2026.13}
}

\supplement{\url{https://doi.org/10.5281/zenodo.21869481}}
\acknowledgements{
We thank Dr. Alba Agustin for sharing the literature benchmarks.
This research was supported by Frequentis and the Austrian Science Fund (FW), grant 10.557766/COE12.
Hecher has been supported by the French National Research Agency (ANR), grant ANR-25-CE23-7647-01, and the Austrian Science Fund (FWF), grant 10.55776/J4656.
}

\nolinenumbers 

\EventEditors{Valentina Cacchiani and Stefan Funke}
\EventNoEds{2}
\EventLongTitle{26th Symposium on Algorithmic Approaches for Transportation Modelling, Optimization, and Systems (ATMOS 2026)}
\EventShortTitle{ATMOS 2026}
\EventAcronym{ATMOS}
\EventYear{2026}
\EventDate{September 3--4, 2026}
\EventLocation{L’Aquila, Italy}
\EventLogo{}
\SeriesVolume{147}
\ArticleNo{13}

\allowdisplaybreaks

\begin{document}

\maketitle

\begin{abstract}
%
While mathematical models act as vital decision support systems for operational Air Traffic Flow and Capacity Management (ATFCM), existing approaches isolate Air Traffic Flow Management (ATFM) from Dynamic Airspace Configuration (DAC).
This separation introduces an unresolved circular dependency between fixed-demand and fixed-capacity assumptions.
Although joint optimization resolves this gap, the enlarged search space renders exact models computationally intractable for medium- to large-scale instances.
To bridge this gap, we propose \emph{ASPaeroFlow}: a heuristic for the joint ATFCM;
it combines instance-space decomposition heuristics with a local exact approach using Answer Set Programming.
We benchmark ASPaeroFlow from small to industry-sized instances and compare it with exact and alternative approaches.
The results indicate that (1) the heuristic provides a computational middle ground between exact methods and operational baselines;
(2) simultaneous optimization can outperform sequential optimization on joint ATFCM;
and (3) an ablation study indicates that DAC has a larger impact on solution quality than flow measures.
\end{abstract}
\section{Introduction}
\label{sec:introduction}

Operational \emph{Air Traffic Flow and Capacity Management} (ATFCM) aims to balance demand and capacity to ensure the safe and efficient flow of air traffic.
Demand is the number of flights intending to traverse through a \textit{sector}, while capacity is the maximum number of flights a sector can safely handle.
Sector capacity is limited by human factors, as \emph{Air Traffic Controllers} (ATCs) are responsible for ensuring separation between aircraft.
\textit{Overloading} a sector means that more flights are in the sector than can be handled by ATCs.
To ensure safety, demand must remain below capacity at all times~\cite{cook_european_2007}.

\emph{Automated decision support systems} help operational ATFCM by optimizing on mathematical models, which suggest aircraft reroutings, delayings, or sector restructurings to optimize airspace usage.
However, the operationally used algorithm \emph{Computer-Assisted Slot Allocation} (CASA), which works according to a ``First-come, First-served'' scheme, has been shown to yield unsatisfactory results compared to more sophisticated optimization models~\cite{eurocontrol_experimental_centre_innovative_2001,bucuroiu_european_2025}.
To circumvent this limit, state-of-the-art (SOTA) work proposes mathematical models for optimizing ATFCM.
Still, SOTA approaches can be largely grouped into one of two classes:
work focusing on the aircraft-level (so-called \emph{Air Traffic Flow Management} (ATFM)) via delays and reroutes~\cite{bertsimas_air_1998,agustin_air_2012,garcia-heredia_combinatorial_2019}
and work focusing on the sector-level (so-called \emph{Dynamic Airspace Configuration} (DAC))~\cite{delahaye_genetic_1994,yin_multi-objective_2016,lui_airspace_2024,chandra_integration_2024}.
Treating ATFM and DAC separately introduces unresolved circular dependencies and makes it hard to compare the actual benefits of ATFM/DAC.
In contrast, \emph{Joint ATFCM} methods~\cite{beiser_LPNMR_2026} would 
(1) resolve the circular dependency between fixed-capacity and fixed-demand assumptions,
(2) render algorithms for isolated subproblems comparable against a joint benchmark,
and (3) enable simultaneous optimization paradigms, which exploit the enlarged search space.
However, while joint ATFCM models exist, their enlarged search space renders them computationally intractable for medium- and large-scale instances when solved exactly.
Therefore, the primary research gap is the lack of \emph{scalable} joint ATFCM methods for industry-sized instances.
  \begin{figure}
    \centering
      \includegraphics[width=14cm]{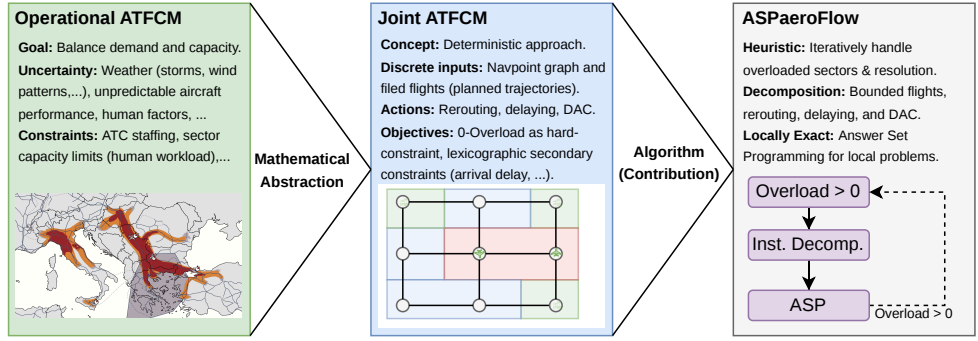}
      \caption{
        Illustrative schematics of the problem setting including the main contribution.
      }
      \label{fig:schematic-concept}
  \end{figure}

\subparagraph{Contributions.}
This paper partially addresses this gap by proposing \emph{ASPaeroFlow}: a heuristic for \emph{joint ATFCM} which optimizes ATFM and DAC simultaneously.
See Figure~\ref{fig:schematic-concept} for the conceptual steps.
In detail, our contributions are as follows:
\begin{itemize}
    \item \textbf{ASPaeroFlow.}
    %
    We propose heuristics that derive 0-overload solutions from initial flight schedules and sector configurations.
    The approach employs an iterative instance-space decomposition to partition global airspace overloads into localized, timestep-specific subproblems.
    These tractable instances are solved exactly via the Logic Programming technique Answer Set Programming (ASP)~\cite{eiter_answer_2009},
    leveraging its expressive rule-based modeling and native integration of soft constraints for domain-specific optimization.
    %
    \item \textbf{Data Generator.} We generate realistic industry-sized (large) instances on real-world navpoint graphs for scaling analysis.
    \item \textbf{Benchmarks \& Ablation.} We evaluate ASPaeroFlow in small~\cite{beiser_LPNMR_2026} to large instances, and other instances from the literature~\cite{agustin_air_2012}.
    ASPaeroFlow bridges the scalability challenge of exact models while significantly outperforming operational baselines.
    Further, the experiments suggest that simultaneous optimization outperforms sequential approaches when DAC alone does not suffice,
    and an ablation study highlights DAC as the primary driver for overload reduction.
\end{itemize}

\subparagraph{Related Work.}
\textbf{Dynamic Airspace Configuration (DAC)} constructs airspace from atomic blocks, distinguishing it from free-partitioning DAS~\cite{lui_robust_2025}.
Automated sector design utilizes genetic algorithms, Voronoi diagrams~\cite{sergeeva_3d_2015, yin_multi-objective_2016, xue_airspace_nodate, chandra_integration_2024}, graph methods~\cite{feng_graphdac_2023}, machine learning~\cite{xu_dynamic_2024}, and MIP~\cite{lui_airspace_2024}.
Heuristics for sector design exist:
sectors should be connected; flights should stay in a sector for a minimum amount of time and not re-enter~\cite{sergeeva_3d_2015}.
Joint ATFCM constructs composite sectors (containing multiple atomic sectors) from atomic sectors (which represent exactly one navpoint vertex).
ASPaeroFlow ensures connected sectors.
%
%
\textbf{Flow Optimization (ATFM)} aims at routing network flow (aircraft) optimally, while adhering to side constraints.
The operationally used \emph{Computer Assisted Slot Allocation} (CASA) heuristics work according to the ``First Planned - First Served'' principle and delays aircraft~\cite{eurocontrol_experimental_centre_innovative_2001,bucuroiu_european_2025}.
Optimization models addressing delay and rerouting explore fairness~\cite{bertsimas2016fairness}, lexicographic objectives~\cite{dalmau_multi-objective_2024}, and computational aspects~\cite{berstimas_ATFM_2011, bolic_reducing_2017, dal_sasso_planning_2019, balakrishnan_optimal_nodate}.
It was shown that the restriction to delaying is already $NP$-hard~\cite{bertsimas_air_1998}.
Crucially, most existing methods integrate DAC solely as a fixed input to ATFM~\cite{agustin_air_2012,garcia-heredia_combinatorial_2019};
the recently proposed joint ATFCM model~\cite{beiser_LPNMR_2026} circumvents this,
however, \emph{joint sequential or simultaneous optimization for industry-sized instances remains an open challenge}, as exact methods struggle already on small instances.\\
\textbf{Answer Set Programming (ASP)} is a symbolic-AI paradigm rooted in logic programming offering exact optimization with high interpretability, deployed for example in industrial scheduling~\cite{balazova_smart_2025, comploi-taupe_interactive_nodate, abels_train_2021}.
Regarding aviation, proposals for the usage of ASP for efficient drone scheduling~\cite{nguyen_optimized_2024} or optimization for mission schedules on vertiports~\cite{kim_enhancing_2025} exist.
We use ASP for several reasons, including its ``natural'' problem modeling, rapid prototyping, and future Explainable AI (XAI) integration potential.
To circumvent ASP's primary scalability limitation --- the \emph{grounding bottleneck}~\cite{semmelrock_investigating_2025} --- we embed it within a problem-decomposition heuristic~\cite{el-kholany_problem_2022}.

\smallskip
Next, we present the necessary preliminaries (Section~\ref{sec:preliminaries}), followed by the introduction of the ATFCM model in Section~\ref{sec:atfcm-model} and ASPaeroFlow (Section~\ref{sec:asp-aero-flow}).
%
%
%
We present the experiments and their results in Section~\ref{sec:benchmarks-and-experiments} and then conclude (Section~\ref{sec:conclusion}).

\section{Preliminaries}
\label{sec:preliminaries}
\subparagraph{Graph Theory.}
We consider weighted undirected graphs $G = (V,E)$ using standard Euclidean or geodesic distance metrics.
A trajectory $\textit{tr}$ is a simple path augmented with strictly increasing timestamps, yielding a sequence of $(v_i, t_i)$ pairs.

\subparagraph{Answer Set Programming (ASP).}
ASP~\cite{gelfond_logic_2002} is a declarative, logic-based problem-solving paradigm.
We omit formal semantics (e.g., stable models) for brevity, referring to standard literature~\cite{eiter_answer_2009}, and focus on the syntax necessary for our encodings.
An ASP program consists of rules $H \leftarrow B$, where the head $H$ is derived if the body $B$ is satisfied.
Variables are capitalized.
Generating an answer set (solution) to an ASP program is usually done by first \emph{grounding} (instantiating the variables) and then solving.
We can use systems such as Clingo~\cite{gebser_theory_2016} to generate an answer set.
We utilize standard extensions~\cite{gebser_clingo_2019} such as constraints (empty head, which invalidate a solution if the body is true), choice rules (to define search spaces), aggregates, and soft constraints (for lexicographic optimization).
The following snippet illustrates these constructs.\footnote{Note that the predicates used in this snippet (e.g., \texttt{reroute/1}, \texttt{overload/3}) are illustrative; the complete ASP encoding for our model is defined later in Section 4.4. and the appendix.}
\begin{lstlisting}[morecomment={[l]{\%}}, commentstyle=\color{darkgreen}]
% Facts (no variables and no body) describing the input data:
config(0,10). config(1,5). fpl(0,5,1). fplt(0,7,2).
% Rule expressing flights are rerouted unless they choose traj. 0.
reroute(ID) :- flightID(ID), not chtrj(ID,0).
% Choice rule: Choose exactly one trajectory P for each flight ID.
1{chtrj(ID,P):trj(ID,P)}1 :- flightID(ID).
% Rule with Aggregate: Count demand and flag overloads.
overload(SEC,T,LOAD-CAP) :- sec_cap(SEC,CAP,T),#count{ID:sec_f(ID,SEC,T)}=LOAD,LOAD>CAP.
% Constraint: Reject solutions where a flight does not occur.
:- flightID(ID), not flight_occurs(ID).
% Soft Constraint: Minimize overload at priority level 10.
:~ overload(SEC,T,OVER). [OVER@10,SEC,T]
\end{lstlisting}

\section{Joint ATFCM Model}
\label{sec:atfcm-model}

  \begin{figure}
    \centering
      \includegraphics[width=13cm]{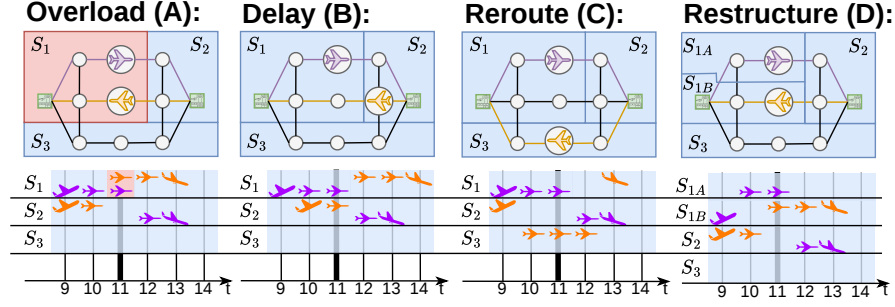}
      \caption{
        Schematics of the ATFCM problem with available actions.
        At time 11 an overload in sector $S_1$ occurs (A). 
        The model allows for three ways to resolve this issue: 
        Apply a delay to a flight (B), reroute a flight (C), or restructure the airspace (D).
      }
      \label{fig:schematic-actions}
  \end{figure}

We present the Joint ATFCM model~\cite{beiser_LPNMR_2026}, where Figure~\ref{fig:schematic-actions} shows our running example. 
\begin{definition}[Navpoint Graph]
    The navpoint graph $G = (V,E)$ is an undirected graph, where each vertex is labeled with spatial geodesic coordinates.
    We distinguish between airports and en-route vertices, 
    where with $\textit{airport}(G)$ and $\textit{en-route}(G)$ we denote the respective sets.
    Let $d : V^2 \rightarrow \mathbb{R}$ be the distance function.
   %
\end{definition}
%
%
\begin{example}
    Figure~\ref{fig:schematic-actions} shows a navpoint graph $G$ (gray).
    Let $G = (V,E)$ with $V = \{a_0, a_1, $ $ v_0, \ldots, v_8\}$, where $\{a_0, a_1\} \subseteq \textit{airport}(G)$ and $\{v_0, \ldots, v_8\} \subseteq \textit{en-route}(G)$.
    Further, $E = \{(a_0,v_0), (a_0,v_3), (a_0,v_6)\} \cup $
    $\{(v_0,v_1), (v_1,v_2), (v_3,v_4), $ $ (v_4,v_5), (v_6,v_7), $ $ (v_7,v_8)\} \cup $ $\{(v_2,a_1),(v_5,a_1),(v_8,a_1)\}$.
    We assume $\forall v_i,v_j \in V: d(v_i,v_j) = 1$ 
\end{example}
%
%
\begin{definition}[Time model]
$T_{\textit{gran}}$ discretizes time into $T_{\textit{gran}}-$timesteps per hour. 
We assume one-day instances; then the set of timesteps $T$ is $T \coloneqq \{t \mid t \in \mathbb{N}, 0 \leq t \leq T_{\textit{gran}} \cdot 24\}$.
\end{definition}
%
%
\begin{example}
\label{ex:graph-model}
    For $T_{\textit{gran}} {=} 4$ we make $15$-minute timesteps, i.e., $4 {\cdot} 24 + 1 = 97 = |T|$ timesteps.
\end{example}
%
%
\begin{definition}[Sector Configuration]
\label{def:DAC-MODEL}
%
A sector $i$ is a set of vertices $\textit{sec}(i,t) \subseteq V$, identified by a representative sector vertex $i \in V$, where $i \in \textit{sec}(i,t)$ ensures the sector index directly corresponds to one of its contained vertices.
ATC-teams have a fixed capacity and have a fixed location, implying atomic capacities of vertices $c_i$ for $i \in V$ ($t=0$): $\textit{cap}(i,0) = c_i$.
A \emph{composite} sector has capacity $\textit{cap}(i,t) = \max_{j \in \textit{sec}(i,t)} \textit{cap}(j,0)$, simulating that the best ATC-team of a region takes command.
To get the sector $i$ for a vertex $v \in V$ for $t \in T$, we define $i = \textit{sec}^{-1}(v,t)$.
\end{definition}
%
%
\begin{example}
\label{ex:DAC-model}
    Ctd. example.
    Static model: $\forall t \in T$. $\textit{sec}(a_0,t) = \{a_0\}$, $\textit{sec}(a_1,t) = \{a_1\}$, 
    $\textit{sec}(v_0,t) = \{v_0, v_1, v_3, v_4\}$ ($S_1$), $\textit{sec}(v_2,t) = \{v_2, v_5\}$ ($S_2$), and $\textit{sec}(v_6,t) = \{v_6,v_7,v_8\}$ ($S_3$).
    Dynamic model, for: $9 \leq t \leq 13$: $\textit{sec}(v_0,t) = \{v_0,v_1\}$ ($S_{1A}$) and $\textit{sec}(v_3,t) = \{v_3,v_4\}$ ($S_{1B}$).
    Atomic capacity: $\forall v \in V:\textit{cap}(v,0) = 1$.
    Static model:$\forall t > 0$: $\textit{cap}(a_0,t) = \textit{cap}(a_1,t) = \textit{cap}(v_0,t) = \textit{cap}(v_2,t) = \textit{cap}(v_6,t) = 1$, with a combined capacity of $5$.
    Dynamic model ($\mathbf{SEC}$): for $9 \leq t \leq 13$ the combined capacity is $6$ (increase of capacity through splitting).
\end{example}
%
%
%
\begin{definition}[Flight Model]
Flights $f \in F$ are modeled as trajectories on the navpoint graph.
Let $f \in F$ be a flight and $f_{id}$ be its flight id: $f = (id_f,((v_0, t_0),(v_1,t_1), \ldots, (v_n,t_n)))$.
\end{definition}
%
%
\begin{definition}[Aircraft Model]
We map flights to aircraft (as a single physical aircraft can have multiple flights per day, meaning a delay in an early flight can propagate downstream): an aircraft $a \in \texttt{A}$ is a three tuple $(id_a, v_a, F_a)$, where $id_a$ is the unique id, $v_a$ is the velocity, and $F_a$ is its set of flights.
%
\end{definition}
%
%
\begin{example}
\label{ex:flight-model}
Ctd. example.
 Let $A = \{p_0,p_1\}$ with $p_0 = (id_{p_0}, v_{p_0}, \{f_0\})$ and $p_1 = (id_{p_1},v_{p_1}, \{f_1\})$.
 Let $f_0,f_1 \in F$ with $f_0 = (0,((a_0,9),(v_0,10),$ $(v_1,11),$ $(v_2,12),(a_1,13)))$ and 
 $f_1 = (1,((a_1,9),(v_5,10),$ $(v_4,11),(v_3,12),(a_0,13)))$.
\end{example}
%
%
\begin{definition}[$\ATFCM$ Instance]
An ATFCM instance is $I = (G, T_{\textit{gran}}, T, \textit{sec}, \textit{cap}, \texttt{A})$,
where $G$ is a navpoint graph,
$T_{\textit{gran}}$ defines the time granularity,
$T$ the timesteps,
$\textit{sec}$ the initial sector configuration,
$\textit{cap}$ the atomic capacities,
and $\texttt{A}$ the aircraft.
We denote with $F$ the set of filed flights.
All ids are unique.
    %
\end{definition}
%
%
\begin{definition}[ATFCM Solution]
\label{def:atfcm-solution}
The ATFCM solution is given by: $\mathbf{S} = (\mathbf{T}, \mathbf{SEC}, \mathbf{A})$,
where $\mathbf{T}$ are the solution timesteps,
$\mathbf{SEC}$ is the solution DAC, and
$\mathbf{A}$ are the solution aircraft.
$\mathbf{T}$ might be expanded beyond the initial $T$ (an instance is always solvable).
The flown flights $f \in \mathbf{F} = \bigcup_{(id_a,v_a,F_a) \in \mathbf{A}} F_a$ may be delayed.
If a flight is delayed and/or rerouted, this is indicated by $r_f=1$.
Sectors in $\mathbf{SEC}$ are allowed to be changed. 
We require that airport sectors remain atomic and en-route sectors remain connected to partially accommodate operational requirements.
\noindent
\textbf{Demand}.
Let $f \in \mathbf{F}$ be a flight reaching navpoint $v_a \in V$ at time $t_a$ and then flying to navpoint $v_b \in V$, which it reaches at time $t_b$.
Assuming that $\Delta t = t_b - t_a$, then $f$ is in the timespan $t \in [t_a, t_a + \frac{\Delta t}{2}]$ in sector $s_i = \textit{sec}^{-1}(v_a,t)$ ($\textit{over}(f,s_i,t) = 1$),
and for $t' \in [t_a + \frac{\Delta t}{2} + 1, t_b]$ in sector $s_{i+1} = \textit{sec}^{-1}(v_b,t')$ ($\textit{over}(f,s_{i+1},t') = 1$).
For a time $t \in T$, and a sector $i \in V$, we measure demand $q(i,t)$, as the number of aircraft currently in the sector, i.e.,  $q(i,t) = \sum_{f \in \mathbf{F}} \textit{over}(f,i,t)$.
For each timestep $t \in \mathbf{T}$, we define with $\mathbf{SEC}$ the DAC.
We require for a solution $\mathbf{S}$ to be accepted that for all $t \in \mathbf{T}$, $i \in V$, demand $q(i,t) \leq \textit{cap}(i,t)$.
\end{definition}
%
%
\begin{example}
Overload (A): Let $\mathbf{S} = (T,\textit{sec},A)$, then $q(v_0,11) = 2 > \textit{cap}(v_0,11) = 1$.
Delay (B): Let $\mathbf{S} = (T,\textit{sec},\mathbf{A})$ with $\mathbf{A}$, where $\mathbf{F} = \{f_0,f_1'\}$ s.t. $f_1' = (1,((a_1,10),(v_5,11),(v_4,12),$ $(v_3,13),(a_0,14)))$.
Reroute (C): Let $\mathbf{S} = (T,\textit{sec},\mathbf{A})$ with $\mathbf{A}$, where $\mathbf{F} = \{f_0,f_1'\}$ s.t. $f_1' = (1,((a_1,9),$ $(v_8,10),(v_7,11),(v_6,12),(a_0,13)))$.
Reconfig (D): Let $\mathbf{S} = (T,\mathbf{SEC},A)$ with $\mathbf{SEC}$ (Ex.~6).
\end{example}
%
%
%
\begin{definition}[Optimization Problem]
\label{def:optimization-problem}
A solution $\mathbf{S}$ is optimized according to efficiency (Arrival Delay, Active en-route sectors, and Navpoint-sector changes) and stability/fairness (Reroutes/delays and Reconfigurations) soft constraints.
An accepted solution $\mathbf{S}$ is strictly preferred to $\mathbf{S}'$ ($\mathbf{S} \prec \mathbf{S}'$) if its objective tuple is lexicographically smaller: 
$(\ATFCMoDELAY, \ATFCMoNUMSEC, \ATFCMoSECCHANGES, \ATFCMoREROUTED, \ATFCMoRECONFIG)_{\mathbf{S}} \prec_{\text{lex}} (\dots)_{\mathbf{S}'}$.
Let $t_f^a$ and $t_f^e$ be the actual and expected arrival times.
The individual criteria to be minimized are defined as follows, utilizing indicator function~$\mathbb{I}(\cdot)$:

\begin{align*}
    \ATFCMoDELAY &= \sum_{f \in \mathbf{F}} d_f  = \sum_{f \in \mathbf{F}} t^a_f - t^e_f && \text{(Arrival delay)} \\
    \ATFCMoNUMSEC &= \sum_{t \in \mathbf{T}_{>0}} \sum_{i \in V} \mathbb{I}\big(|\mathbf{SEC}(i,t)| > 0\big) && \text{(Active Sectors)} \\
    \ATFCMoSECCHANGES &= \sum_{t \in \mathbf{T}_{>1}} \sum_{v \in V} \mathbb{I}\big(\mathbf{SEC}^{-1}(v,t-1) \neq \mathbf{SEC}^{-1}(v,t)\big) && \text{(Navpoint-sector changes)} \\
    \ATFCMoREROUTED &= \sum_{f \in \mathbf{F}} r_f && \text{(Reroutes/delays)} \\
    \ATFCMoRECONFIG &= \sum_{t \in \mathbf{T}_{\geq 0}} \sum_{v \in V} \mathbb{I}\big(\textit{sec}^{-1}(v,t) \neq \mathbf{SEC}^{-1}(v,t)\big) && \text{(Reconfigurations)}
\end{align*}
where $\mathbf{SEC}^{-1}(v, t)$ returns the representative sector index $i$ containing vertex $v$ at time $t$.
\end{definition}

\begin{table}[t]
\caption{
Representative parameter bounds used in $ASPaeroFlow_{r,d}$'s local ASP subproblems. Exact values are variant dependent (cf.\ Section~\ref{sec:benchmarks-and-experiments}).
}
\label{tab:bounds-decomposition}
\centering
\tiny
\setlength{\tabcolsep}{5pt}
\renewcommand{\arraystretch}{1.15}
\begin{tabular}{p{0.10\linewidth} p{0.68\linewidth} p{0.15\linewidth}}
\toprule
\multicolumn{3}{l}{\textbf{Chosen bounds in instance-space decomposition}}\\
\midrule
Parameter & Operational rationale & $\text{ASPaeroFlow}_{r,d}$ \\
\midrule
$f_{\max}$ &
Limit the local subproblem to flights that actually contribute to the current overload; keeps grounding and solving stable while capturing the dominant contributors. & $2$ flights\\
$d_{\max}$ &
Strategic regulations usually consider bounded departure shifts; a short delay window provides a realistic control measure without exploding combinations. & $5$ steps\\
$r_{\max}$ &
A small set of diverse candidate routes (e.g., $k$-shortest with similarity filtering) reflects that only few reroute options are operationally acceptable and keeps the action set manageable. & 
$3$ routes \\
$p_{\max}$ &
Use a small set of sector split options for the overloaded sector; captures the main capacity measure (partitioning) while avoiding enumeration of arbitrary partitions. &
$2$ splits \\
\bottomrule
\end{tabular}
\end{table}

\section{New Decomposition Heuristics}
\label{sec:asp-aero-flow}
Although some large-scale ATFM models can be solved exactly by exploiting inherent structures~\cite{balakrishnan_optimal_nodate},
we expect that integrating DAC introduces dense coupling constraints. These constraints likely create structural difficulties that prevent scaling exact solution approaches to medium- and large-scale instances.
Moreover, the required sequential bounding for the five-level lexicographic objective severely degrades solver performance.
Consequently, exact methods suffer from combinatorial explosion, rendering multi-objective exact solving computationally intractable even for relatively small instances~\cite{beiser_LPNMR_2026}.
To ensure operational feasibility, we propose a heuristic approach:
\emph{ASPaeroFlow}\footnote{GitHub Link: \url{https://github.com/alexl4123/ASPaeroFlow-Optimizer}}.
ASPaeroFlow decomposes the problem around the next overloaded sector, building a bounded local subproblem containing only the most relevant flights and the overloaded sector.
Specifically, we cap (i) the number of considered flights ($f_{\max}$), (ii) the delay window per flight ($d_{\max}$), (iii) the number of alternative routes per flight ($r_{\max}$), and (iv) the number of partition options for the overloaded sector ($p_{\max}$); see also Table~\ref{tab:bounds-decomposition}.
An exact solver then selects the best local solution under our lexicographic objective, and the procedure iterates until all overloads are resolved.
The chosen local bounds balance search-space exploration with solver runtimes.
Empirically, increasing the flight limit ($f_{\max} > 2$) or the partitioning limit ($p_{\max} > 2$) drastically increases solving times.
Conversely, the routing ($r_{\max}$) and delay ($d_{\max}$) bounds are computationally more forgiving; the chosen parameters were determined via empirical parameter tuning to ensure rapid local convergence while maintaining high solution quality.

\subsection{Optimization Loop}

Algorithm~\ref{alg:heuristic-alg-1} shows the overall algorithm.
In Lines~(\ref{alg:heuristic-alg-1:line-1})--(\ref{alg:heuristic-alg-1:line-4}) we do pre-processing steps, by converting the instance to the internal matrix representation $\mathbf{M}$ and computing the overloads.
We consider as overloads the sum of exceedances: $K_o = \sum_{t \in T, i \in V} \max\{0,q(i,t) - \textit{cap}(i,t)\}$. 
The algorithm iteratively solves overloads until $K_o = 0$ (Line~(\ref{alg:heuristic-alg-1:line-5})).
The decomposed instance $I$ is computed by instance space decomposition (Line~(\ref{alg:heuristic-alg-1:line-6}) --- Algorithm~\ref{alg:instance-decomp}),
which is then given with the encoding $E$ to the local optimization procedure (Line~(\ref{alg:heuristic-alg-1:line-7})).
The local optimization procedure computes a new candidate solution $\mathbf{M}'$.
A \emph{correct local optimization encoding} selects among its available choices the solution that reduces the magnitude of overload the most,
followed by adhering to the criteria in Definition~\ref{def:optimization-problem}.
The candidate solution $\mathbf{M}'$ is accepted if it has less overloads (Lines~(\ref{alg:heuristic-alg-1:line-10})--(\ref{alg:heuristic-alg-1:line-13})).
Otherwise, the parameters of the decomposition are adapted to allow for other solutions to be explored to ensure termination (Line~(\ref{alg:heuristic-alg-1:line-15})).
%

\subsection{Instance Space Decomposition}

\begin{figure}[t]
    \begin{algorithm}[H]
        \small
        \KwData{Instance $\mathcal{I}$, Encoding $E$, Parameters $\mathcal{P}$}
        $\mathbf{M} \leftarrow \textit{convertInstance}(\mathcal{I})$ ; 
        \label{alg:heuristic-alg-1:line-1}
        \qquad $\textit{overloads} \leftarrow \textit{computeOverloads}(\mathbf{M})$ \; 
        $\mathcal{P}' \leftarrow \textit{copy}(\mathcal{P})$ \;\label{alg:heuristic-alg-1:line-4}
        \While{$\textit{overloads} > 0$} {\label{alg:heuristic-alg-1:line-5}
            $I \leftarrow \textit{decomposeInstance}(\mathbf{M}, \mathcal{P}')$ \; \label{alg:heuristic-alg-1:line-6}
            $\mathbf{M}' \leftarrow \textit{LocalOptimization}(I,E)$ ; \label{alg:heuristic-alg-1:line-7}
            %
            \qquad $\textit{overloads}' \leftarrow \textit{computeOverloads}(\mathbf{M}')$ \;
            \uIf{$\textit{overloads}' < \textit{overloads}$} { \label{alg:heuristic-alg-1:line-10} 
                 $\mathbf{M} \leftarrow \mathbf{M}'$ ;
                 \qquad $\textit{overloads} \leftarrow \textit{overloads}'$ \;
                 $\mathcal{P}' \leftarrow \textit{copy}(\mathcal{P})$ \; \label{alg:heuristic-alg-1:line-13}
            }\uElseIf{$\textit{exhausted}(\mathcal{P}')$} {
                    \Return $(\mathbf{M}, \text{\textbf{False}})$ \tcp*{Termination with residual overload} \label{alg:heuristic-alg-1:line-exhausted}
            }\Else{
                    $\mathcal{P}' \leftarrow \textit{adjustParameters}(\mathcal{P}')$ \; 
                    \label{alg:heuristic-alg-1:line-15}
            }
        }
        \Return $(\mathbf{M}, \text{\textbf{True}})$ \tcp*{Successful termination} \label{alg:heuristic-alg-1:line-return}
       \caption{ASPaeroFlow}
        \label{alg:heuristic-alg-1}
    \end{algorithm}
\end{figure}

Algorithm~\ref{alg:instance-decomp} shows our instance decomposition approach.
We reduce the search space, by limiting the number of investigated flights $f_{\max}$, the delay $d_{\max}$, the alternative routes $r_{\max}$, and the possible DAC partitions $p_{\max}$ (Line~(\ref{alg:instance-decomp:line:1})).
Line~(\ref{alg:instance-decomp:line:2}) obtains the time $t$ and sector $s$ of the first overloaded sector (lexicographically time and sector-wise).
In Line~(\ref{alg:instance-decomp:line:3})
we limit the search space to $f_{\max}$ flights.
The next part in Lines~(\ref{alg:instance-decomp:line:6})--(\ref{alg:instance-decomp:line:15}) is to generate the instance for each of the $f_{\max}$ flights, by computing the alternatives $r_{\max}$ and considering the maximum delay $d_{\max}$.
This is followed by the generation of the instance for DAC, where we consider $p_{\max}$ partitions (Lines~(\ref{alg:instance-decomp:line:16})--(\ref{alg:instance-decomp:line:19})),
ranging from the current configuration $p = 1$ to a partition of the sector $s$ into $p_{\max}$ parts.
Finally, we combine the instance in Line~(\ref{alg:instance-decomp:line:21}) and return it.


\subparagraph{Reroute Generation.}
We compute the $r_{\max}$ routes (Algorithm~\ref{alg:instance-decomp} in Line~\ref{alg:instance-decomp:line:9}), using an iterative algorithm that computes the shortest routes on $G$ between departure and destination airports.
On each iteration the current shortest path $p$ is penalized and a diverse set of routes is ensured by only accepting a path $p$ if its Jaccard similarity to any other already considered path $p'$ is small enough: $\frac{p' \cap p}{p' \cup p} < j_m$, where $j_m = 0.6$. 
Let $w(p)$ be the path weight, then each edge $e$ of $e \in p$ is penalized by $w_{t+1}(e) \leftarrow w_{t}(e) + 0.1 \frac{w(p)}{|p|}$,
where $w_0(e) = d(e)$.

\subparagraph{Alternative Airspace Configurations.}
The $p_{\max}$ graph partitions are computed in Algorithm~\ref{alg:instance-decomp} in Line~(\ref{alg:instance-decomp:line:18}) by an iterative breadth-first search algorithm.
We partition a sector $s$ into $p$ sectors ($\{s_1,\ldots,s_p\} = S$) of approximate size $\frac{|\textit{sec}(s,t)|}{p}$ iteratively.
Assuming we generated $|S| = i-1$ sectors so far, we add a new sector $s_i$ via two steps:
first, we select a $\textit{seed} \in \textit{sec}(s,t) \setminus \left( \cup_{s_j \in S} \textit{sec}(s_j,t) \right)$ vertex and add it to the new sector $s_i = \{\textit{seed}\}$.
Second, we add vertices $v \in \textit{sec}(s,t) \setminus \left( \cup_{s_j \in S} \textit{sec}(s_j,t) \right)$ to $s_i$ in a breadth-first manner (ensuring connection of $s_i$) until $|\textit{sec}(s_i,t)| > \frac{|\textit{sec}(s,t)|}{p}$; then we add $s_i$ to $S$.
%
\subparagraph{Adjust Parameters.}
If search stalls, \texttt{adjustParameters} expands the temporal horizon using a rolling delay ($+5$ timesteps).
After $10$ non-improvement steps, the flight limit reduces to $f_{\max} = 1$.
Assuming any flight can be routed in isolation, the algorithm \emph{terminates} and returns $\mathbf{M}$ (Line~\ref{alg:heuristic-alg-1:line-exhausted}) 
if $f_{\max} = 1$, the subsequent step yields no improvement, and the single flight $f$ remains unroutable due to absolute capacity limits (e.g., zero-capacity sectors).

\begin{algorithm}[t]
    \small
    \KwData{Converted Instance $\mathbf{M}$, Parameters $\mathcal{P}$}
    $f_{\max}, d_{\max}, r_{\max}, p_{\max} \leftarrow \mathcal{P}$ \; \label{alg:instance-decomp:line:1}
    $s,t \leftarrow \textit{computeFirstOverload}(\mathbf{M})$ \; \label{alg:instance-decomp:line:2}
    $F \leftarrow \textit{flightsInOverload}(\mathbf{M}, s, t)$ $\qquad$  $F \leftarrow \textit{sort}(F)$ $\qquad$ $F \leftarrow F[:f_{\max}]$  \; \label{alg:instance-decomp:line:3}
    %
    %
    $\textit{I}_F \leftarrow \emptyset$ \; \label{alg:instance-decomp:line:6}
    \For{$f \textit{ in } F$} { \label{alg:instance-decomp:flights-pt1}
        $\textit{config} \leftarrow 0$ \;
        $\textit{R} \leftarrow \textit{bestRoutes}(f, G, r_{\max})$ \; \label{alg:instance-decomp:line:9}
        \For{$d \textit{ in range}(d_{\max}+1)$}
        {
            \For{$r \textit{ in } R$} {
                $f' \leftarrow \textit{computeNavpointFlightPlan}(f,d,r)$ \;
                $\textit{I}_F \leftarrow I_F \cup \{f'\}$ \;
                $\textit{config} \leftarrow \textit{config} + 1$ \; \label{alg:instance-decomp:line:15}
            }
        } 
    }
    $\textit{I}_S \leftarrow \emptyset$ \; \label{alg:instance-decomp:line:16}
    \For{$p \textit{ in range}(1,p_{\max}+1)$} {
        $\textit{partition}_p \leftarrow \textit{computePartition}(s,p,G)$ \; \label{alg:instance-decomp:line:18}

        $\textit{I}_S \leftarrow I_S \cup \{\textit{partition}_p\}$ \;      \label{alg:instance-decomp:line:19}
    }
    $\textit{I} \leftarrow \textit{I}_F \cup \textit{I}_S$ \; 
    \label{alg:instance-decomp:line:21}
    \Return $\textit{I}$ \;
   \caption{Instance Space Decomposition}
    \label{alg:instance-decomp}
\end{algorithm}

\subsection{Correctness}

We provide a brief argument for correctness of the algorithm.
Let $\mathcal{I}$ be an instance of the optimization problem and let $\mathcal{S}$ be the set of solutions.
For the following proofs, we take the operational assumption that all filed flights can be flown in principle;
meaning that $\forall v \in V:$ $\textit{cap}(v,0) \geq 1$.
We refer to this as \emph{operational feasibility}.
Operational feasibility guarantees that a single flight from an overloaded sector can always be delayed to reduce the magnitude of overload---consequently, if delaying the single flight does not lead to a reduction in overload magnitude, the algorithm can safely terminate.
\begin{theorem}
    Let $\mathcal{I}$ be an instance of the optimization problem and let the operational feasibility assumption hold.
    Then, for any correct encoding E, if ASPaeroFlow (Algorithm~\ref{alg:heuristic-alg-1}) produces an output, its output is a valid solution.
\end{theorem}
%
%
\begin{proof}[Proof (Sketch).]
    The Algorithm~\ref{alg:heuristic-alg-1} returns the solution $\mathbf{M}$ in Line~\ref{alg:heuristic-alg-1:line-return},
    or in Line~\ref{alg:heuristic-alg-1:line-exhausted}.
    %
    We prove that with the operational feasibility assumption, if Algorithm~\ref{alg:heuristic-alg-1} terminates, it terminates by returning the solution $\mathbf{M}$ in Line~\ref{alg:heuristic-alg-1:line-return}:
    Towards a contradiction assume Algorithm~\ref{alg:heuristic-alg-1} terminates prematurely on exhaustion (Line~\ref{alg:heuristic-alg-1:line-exhausted}).
    For this, it must be the case that for at least $10$ subsequent steps the algorithm did not experience a reduction of overload, although it increased the time window start to $t_s$ and reduced the number of considered aircraft to $f_{\max} = 1$.
    Let the considered flight for optimization be $f = (id,tr)$ for sector $i \in V$ and time $t' \in \mathbf{T}$.
    Further, let $f'$ be the flight in the current solution $\mathbf{M}$ with the maximum actual arrival delay $t^a_{f'}$.
    Then, to exit prematurely on exhaustion in Line~\ref{alg:heuristic-alg-1:line-exhausted}, flight $f$ cannot be scheduled after flight $f'$,
    so it must hold that whenever $t_s > t^a_{f'}$ no improvement is found.
    However, such an improvement is found by local optimization due to the operational assumption:
    adjust $tr$ to $tr_s$, by keeping the path, but starting at $t_s$, the overload of $i$ decreases and no increase in overload of any other sector is experienced,
    as by the operational feasibility assumption $\forall (v,t) \in tr_s: \textit{cap}(v,t) \geq q(v,t) = 1$.
    This is a contradiction to the assumption.

    It remains to argue that the returned $\mathbf{M}$ from Line~\ref{alg:heuristic-alg-1:line-return} is indeed a valid solution (Def.~\ref{def:atfcm-solution}).
    First, there are no overloads remaining in $\mathbf{M}$ due to Line~\ref{alg:heuristic-alg-1:line-5}.
    Further, $\mathbf{A}$ and $\mathbf{SEC}$ are valid.
    $\mathbf{A}$ is valid, as (a) the instance space decomposition generates valid trajectories and ensures that for each considered flight its subsequent flights are also part of the instance passed to local optimization and thereby cannot overlap (Lines~(\ref{alg:instance-decomp:flights-pt1})--(\ref{alg:instance-decomp:line:15}) of Algorithm~\ref{alg:instance-decomp})
    and (b) local optimization ensures that exactly one trajectory is selected for each flight.
    %
    Further, $\mathbf{SEC}$ is valid as bounded connected partitions introduced by Lines~(\ref{alg:instance-decomp:line:16})--(\ref{alg:instance-decomp:line:19}) of Algorithm~\ref{alg:instance-decomp}, combined with local optimization ensure valid partitions.
\end{proof}

\begin{observation}
    Let $\mathcal{I}$ be an instance of the optimization problem.
    For any correct encoding E without assuming operational feasibility, Algorithm~\ref{alg:heuristic-alg-1} may return an incorrect solution.
\end{observation}
\begin{proof}[Proof (Sketch).]
    When the operational feasibility assumption is violated, it is not guaranteed that a flight $f$ can actually occur;
    however, it can be the case that ASPaeroFlow terminates prematurely on exhaustion (Line~\ref{alg:heuristic-alg-1:line-exhausted}) in Algorithm~\ref{alg:heuristic-alg-1},
    thereby producing an invalid solution $\mathbf{S}$, although in principle a solution $\mathbf{S}'$ exists by Def.~\ref{def:atfcm-solution}:
    this can happen by (1) ASPaeroFlow not considering routes, which are outside of the bounded rerouting $r_{\max}$ options,
    or (2) ASPaeroFlow not considering sector configurations, which are not evaluated due to bounded sectorization $p_{\max}$.
\end{proof}

\begin{theorem}
    Let $\mathcal{I}$ be an instance of the optimization problem.
    Then, for any correct encoding E, ASPaeroFlow (Algorithm~\ref{alg:heuristic-alg-1}) terminates.
\end{theorem}
\begin{proof}[Proof (Sketch).]
    In each iteration of Line~\ref{alg:heuristic-alg-1:line-5} of Algorithm~\ref{alg:heuristic-alg-1} either the overload strictly decreases (a new solution is accepted as the current one and the parameter bounds are reset to their initial values),
    or the parameter bounds are expanded via \textit{adjustParameters}.
    If the total overload reaches $0$, the algorithm terminates successfully (Line~\ref{alg:heuristic-alg-1:line-return}).
    Otherwise, the algorithm ensures $f_{\max} = 1$ by setting it after $10$ non-improvement steps via \textit{adjustParameters}.
    The \textit{exhausted} condition then becomes true if the earliest starting time $t_f$ of the single ($f_{\max} = 1$) considered flight $f$ is strictly later than the arrival time $t^a_{f'}$ of any other flight $f'$: $t_f > t^a_{f'}$ (and no improvement occurs).
    When exhausted, the algorithm terminates with a residual overload (Line~\ref{alg:heuristic-alg-1:line-exhausted}).
\end{proof}

%
\subsection{Local Optimization via ASP}

We proceed to describe an encoding in ASP. Conceptually, ASP is similar to Boolean satisfiability,
but allows for more natural problem modeling, rapid integration of changed problem definitions, and directly provides minimization statements.

For brevity, the detailed generation of facts, trajectory guesses, and structural mappings are deferred to the appendix.
Here, we focus on the optimization objectives.
We treat the minimization of the magnitude of overload as the primary objective -- as the ASP model operates locally and cannot solve global overloads.
Besides that, we minimize lexicographically according to the \ATFCMOPT problem.
Lines~(1) and~(2) describe minimization of overloads, while Line~(3) describes minimization of arrival delay.
By construction, Line~(4) minimizes both the number of sectors and navpoint-sector changes.
Finally, Line~(5) minimizes reroutes/delays and Line~(6) minimizes divergences from $\textit{sec}$.
\begin{lstlisting}
:~ overload(X,T,OVER). [OVER@10,X,T]
:~ chc(CONFIG), config(CONFIG,TOTAL_OVER). [TOTAL_OVER@10,CONFIG]
:~ arrival_delay(ID,DIFF). [DIFF@9,ID]
:~ chc(CONFIG), config_number_sectors(CONFIG,NUMBER). [NUMBER@8,CONFIG]
:~ reroute(ID). [1@7,ID]
:~ chc(CONFIG). [CONFIG@6,CONFIG]
\end{lstlisting}

\begin{lemma}
    The ASP encoding is a correct local optimization encoding.
\end{lemma}
\begin{proof}[Proof (Sketch).]
    We argue for two properties:
    (a) the found solution $\mathbf{M}'$ locally adheres to the specifications of an 
    ATFCM solution (Def.~\ref{def:atfcm-solution}) except the overload;
    and (b) the found solution $\mathbf{M}'$ is the locally optimal s.t. it reduces overloads while adhering to side constraints (Def.~\ref{def:optimization-problem}).
    (a): To prevent invalid double assignments or double sector splits, exactly one path ($\texttt{chtrj(ID,P)}$) and one configuration ($\texttt{chc(CONFIG)}$) is selected.
    (b): the primary optimization criteria is to reduce $K_o$, as finding a globally $K_o=0$ solution by local optimization cannot be guaranteed;
    the objective combines local effects computed in ASP (\texttt{overload(X,T,OVER)})
    with effects computed by Algorithm~\ref{alg:instance-decomp} (\texttt{config(CONFIG,TOTAL\_OVER)}).
    The other optimization criteria adhere to the specification of Def.~\ref{def:optimization-problem}:
    for \texttt{[NUMBER@8,CONFIG]} we use the fact that for any two configurations produced by instance space decomposition, if the number of resulting splits is lower (\ATFCMoNUMSEC), the number of sector changes is as well (\ATFCMoSECCHANGES).
    The number of reroutes \ATFCMoREROUTED directly translates from Def.~\ref{def:optimization-problem}; the number of reconfigurations is minimized by the fact that that \texttt{chc(CONFIG)} for \texttt{chc(0)} has $0$ introduced config changes.    
\end{proof}

\begin{table}[t]
\caption{Compared variants and their action bounds.
Global approach search space $\approx d_{\max}^{|F|} \cdot r_{\max}^{|F|} \cdot p_{\max}^{|V|}$,
local search space (per iteration, assuming decomposition $f_{\max}=2,|V|=1$) $\approx d_{\max}^{2} \cdot r_{\max}^{2} \cdot p_{\max}$.
}
\label{tab:variants}
\centering
\tiny
\setlength{\tabcolsep}{2pt}
\renewcommand{\arraystretch}{0.95} %
\begin{tabular}{lccc p{0.60\linewidth}}
\toprule
Variant & $p_{\max}$ & $r_{\max}$ & $d_{\max}$ & Notes \\
\midrule
\multicolumn{5}{l}{\textbf{Basic Baseline}} \\
\midrule
Initial & 0 & 0 & 0 & Result if no action is taken (naive baseline).\\
\midrule
\multicolumn{5}{l}{\textbf{Global (Exact) Approaches}~\cite{beiser_LPNMR_2026}} \\
\midrule
ASP-P & 4 & 4 & expanded & Global decisions; partial DAC; partial Rerouting, partial delaying; delays expanded up to feasibility; initially $d_{\max} \approx T_{\textit{gran}} \cdot 24$ \\
MIP & 1 & 2 & expanded & Global decisions; partial rerouting; small delay bound iteratively expanded up to feasibility; initially $d_{\max} \approx T_{\textit{gran}} \cdot 24$ \\
\midrule
\multicolumn{5}{l}{\textbf{Local Heuristic Approaches}} \\
\midrule
$\text{ASPaeroFlow}_{r,d}$ & 2 & 3 & 5 & Default configuration. \\
$\text{ASPaeroFlow}_{\neg r,d_p}$  & 2 & 1 & 1 & Only sectorization; if stuck repeatedly delaying possible. \\
$\text{ASPaeroFlow}_{r,d_p}$  & 2 & 1 & 5 & Sectorization and delay (no rerouting).\\
$\text{ASPaeroFlow}_{\neg r,d_p}$  & 2 & 3 & 1 & Sectorization and rerouting; if stuck repeatedly delaying possible.  \\
$\text{ASPaeroFlow}_{seq}$ & 2 & 3 & 5 & Sequential execution of DAC and (full) ATFM. \\
$\text{ATFM}_{r,d}$ & 1 & 3 & 5 & Flow-only ATFM (no DAC). \\
$\text{ATFM}_{r,d_p}$ & 1 & 3 & 1 & Rerouting-only baseline (no delay); if stuck repeatedly delaying possible. \\
$\text{ATFM}_{\neg r,d}$ & 1 & 1 & 5 & Delay-only ATFM (no rerouting). \\
CASA & 1 & 1 & 5 & Single flight delaying (decomposition $f_{\max}=1$). \\
\bottomrule
\end{tabular}
\end{table}

\section{Benchmarks and Experiments}
\label{sec:benchmarks-and-experiments}
To demonstrate the operational feasibility of our approach, we design our computational experiments to evaluate three primary dimensions:
(i) the solution quality of our decomposition heuristic relative to exact models,
(ii) the scalability of the approach compared to related heuristic approaches across progressively larger, real-world-sized scenarios,
and (iii) the effects of integrating DAC into ATFM via an ablation study.
To quantify these dimensions, we extract the following evaluation metrics:
computational performance (execution time and peak RAM usage) and
model metrics (unresolved overloads, \ATFCMoDELAY, \ATFCMoNUMSEC, \ATFCMoSECCHANGES, \ATFCMoREROUTED, \ATFCMoRECONFIG).
%


\subsection{Benchmarked Variants}

We benchmarked $12$ variants.
One (\emph{Initial}) is our baseline, if no action is performed.
The \emph{exact} solvers comprise \emph{ASP-P}, which is $\textit{ASP}_{r_p,d_p,s_p}$ --- a bounded action exact model which balances time to first and time to best solution --- and \emph{MIP}, a mapping of state-of-the-art ATFM MIP formulations to the joint ATFCM model (meaning it selects optimal bounded reroutings and delays, but has no sectorization actions)~\cite{beiser_LPNMR_2026}.
%
%
ASPaeroFlow comprises the remaining nine variants under various parameter configurations.
Variants executing joint optimization of ATFM and DAC are denoted as \emph{ASPaeroFlow}, while restricted variants optimizing only ATFM are denoted as \emph{ATFM}.
Subscripts $r \in \{r,\neg r\}$ and $d \in \{d, d_p\}$ indicate whether rerouting or delaying is permitted.
To simulate operational baselines, we include \emph{CASA}, a greedy ``First-Come, First-Served'' simulation restricting optimization to single-flight delays ($f_{\max} = 1$).
Furthermore, the $\textit{ASPaeroFlow}_{seq}$ \emph{sequential} variant models standard collaborative decision-making by separating the pipeline: optimizing sectorization first, followed by a $\textit{ATFCM}_{r,d}$ flow optimization.

\begin{table}[t]
  \caption{Instance Description: $|V|$ is the sum of navpoint and airport vertices, $|S|$ is the maximum sector size, $C$ the maximum capacity (or range of maximum capacity), $TG$ the time granularity, $|F|$ the number (range) of flights, and $|A|$ the number (range) of aircraft.}
  \label{tab:scenario-description}
  \centering
  \tiny
  \setlength{\tabcolsep}{5.5pt} 
  \renewcommand{\arraystretch}{0.95} %
  \resizebox{\linewidth}{!}{%
    \begin{tabular}{lrrrrcccc}
      \toprule
      \multicolumn{8}{l}{\textbf{Small Instances}~\cite{beiser_LPNMR_2026}} \\
      \midrule
      Instance & Navpoints & Airports & $|V|$ & $|S|$ & $C$ & $T_{\textit{gran}}$ & $|F|$ & $|A|$ \\
      \midrule
    EA-3x3 & 9 & 7 & 16 & 3 & 1 & 1 & 10 to 100 & 6 to 41 \\
    CE-5x5 & 25 & 7 & 32 & 5 & 1 & 1 & 10 to 100 & 6 to 48 \\
    IND-4x10 & 40 & 8 & 48 & 6 & 1 & 1 & 10 to 100 & 6 to 43 \\
    USA-7x7 & 49 & 8 & 57 & 7 & 1 & 1 & 10 to 100 & 5 to 50 \\
    EUR-10x10 & 100 & 17 & 117 & 10 & 1 & 1 & 10 to 100 & 7 to 53 \\
      \midrule
      \multicolumn{8}{l}{\textbf{SOTA Instances}~\cite{agustin_air_2012}} \\
      \midrule
    BLO1 & 107 & 40 & 147 & 8 & 34 & 4 & 1054 & 568 \\
    BLO2 & 1172 & 20 & 1192 & 59 & 88 & 4 & 3196 & 3196 \\
    BLO3 & 2707 & 30 & 2737 & 135 & 130 & 4 & 6475 & 6475 \\
    IC & 5932 & 47 & 5979 & 118 & 25 & 4 & 3157 & 3157 \\
    c01 & 107 & 40 & 147 & 8 & 34 & 4 & 1054 & 568 \\
    c02 & 180 & 40 & 220 & 7 & 30 & 4 & 1330 & 671 \\
    c03 & 147 & 45 & 192 & 8 & 43 & 4 & 1545 & 788 \\
    mBLO1 & 100 & 4 & 104 & 9 & 67 & 4 & 812 & 812 \\
    mBLO3 & 2707 & 30 & 2737 & 135 & 130 & 4 & 6475 & 6475 \\    
      \midrule
      \multicolumn{5}{l}{\textbf{Large Instances}} & $P\cdot 1.0 \rightarrow P \cdot 0.1$ & & & \\
      \midrule
    CE-G & 49 & 7 & 56 & 11 & 3 to 1075 & 4 & 1000 to 31622 & 202 to 5978 \\
    USA-G & 200 & 8 & 208 & 20 & 5 to 1615 & 4 & 1000 to 31622 & 272 to 8647 \\
    EUR-G & 800 & 17 & 817 & 85 & 5 to 1719 & 4 & 1000 to 31622 & 279 to 7506 \\
    EA-G & 1600 & 7 & 1607 & 34 & 4 to 2015 & 4 & 1000 to 31622 & 322 to 14007 \\
    CE & 627 & 10 & 637 & 35 & 11 to 3716 & 4 & 1000 to 31622 & 594 to 19069 \\
    DACH & 1283 & 324 & 1607 & 50 & 3 to 1065 & 4 & 1000 to 31622 & 587 to 18699 \\
    EUR & 8479 & 1588 & 10067 & 50 & 2 to 719 & 4 & 1000 to 31622 & 886 to 28020 \\
    USA & 18101 & 1509 & 19610 & 50 & 2 to 1042 & 4 & 1000 to 31622 & 976 to 30586 \\
      \bottomrule
    \end{tabular}
  }
\end{table}

\begin{table}[t]
%
\caption{Tournament wins $\#W$ (lexicographic: $K_o \succ \ATFCMoDELAY \succ \ATFCMoNUMSEC \succ \ATFCMoSECCHANGES \succ \ATFCMoREROUTED \succ \ATFCMoRECONFIG$) and number of solved $\#S$ instances ($K_o = 0$) by scenario.}
\label{tab:tournament-wins}
\centering
\tiny
  \renewcommand{\arraystretch}{0.95} %
\resizebox{\textwidth}{!}{%
\begin{tabular}{lr|rr|rr|rr|rr|r}
\toprule
Scenario & \#I & \multicolumn{2}{c}{$ASPaeroFlow_{r,d}$} & \multicolumn{2}{c}{$ASPaeroFlow_{seq}$} & \multicolumn{2}{c}{$ASPaeroFlow_{r, d_p}$} & \multicolumn{2}{c}{$ASPaeroFlow_{\neg r, d}$} & Draws \\
 &  & \#W & \#S & \#W & \#S & \#W & \#S & \#W & \#S & \#D \\
\midrule
Large Instances (PCAP > 0.5) & 840 & 46 & 834 & \textbf{343} & \textbf{840} & 127 & 831 & 1 & 826 & 250 \\
Large Instances (PCAP <= 0.5) & 840 & 211 & 674 & 213 & \textbf{685} & \textbf{290} & 605 & 6 & 631 & 82 \\
Small Instances & 200 & 15 & \textbf{200} & 13 & \textbf{200} & 51 & \textbf{200} & 0 & \textbf{200} & 8 \\
SOTA Instances & 9 & 0 & \textbf{9} & 1 & \textbf{9} & 1 & 7 & 0 & \textbf{9} & 2 \\
\midrule
Overall & 1889 & 272 & 1717 & \textbf{570} & \textbf{1734} & 469 & 1643 & 7 & 1666 & 342 \\
\addlinespace[0.6em]
Scenario & \#I & \multicolumn{2}{c}{$ASPaeroFlow_{\neg r, d_p}$} & \multicolumn{2}{c}{$ATFM_{r,d}$} & \multicolumn{2}{c}{$ATFM_{r, d_p}$} & \multicolumn{2}{c}{$ATFM_{\neg r, d}$} & Draws \\
 &  & \#W & \#S & \#W & \#S & \#W & \#S & \#W & \#S & \#D \\
\midrule
Large Instances (PCAP > 0.5) & 840 & 52 & 815 & 2 & 824 & 4 & 754 & 0 & 820 & 250 \\
Large Instances (PCAP <= 0.5) & 840 & 37 & 525 & 0 & 407 & 0 & 205 & 0 & 411 & 82 \\
Small Instances & 200 & 6 & \textbf{200} & 0 & \textbf{200} & 1 & \textbf{200} & 0 & \textbf{200} & 8 \\
SOTA Instances & 9 & 0 & 7 & \textbf{3} & \textbf{9} & 2 & 7 & 0 & 8 & 2 \\
\midrule
Overall & 1889 & 95 & 1547 & 5 & 1440 & 7 & 1166 & 0 & 1439 & 342 \\
\addlinespace[0.6em]
Scenario & \#I & \multicolumn{2}{c}{CASA} & \multicolumn{2}{c}{Initial} & \multicolumn{2}{c}{MIP} & \multicolumn{2}{c}{ASP-P} & Draws \\
 &  & \#W & \#S & \#W & \#S & \#W & \#S & \#W & \#S & \#D \\
\midrule
Large Instances (PCAP > 0.5) & 840 & 0 & 824 & 0 & 168 & 15 & 180 & 0 & 112 & 250 \\
Large Instances (PCAP <= 0.5) & 840 & 0 & 448 & 0 & 0 & 1 & 1 & 0 & 3 & 82 \\
Small Instances & 200 & 4 & \textbf{200} & 0 & 1 & 19 & 82 & \textbf{83} & 92 & 8 \\
SOTA Instances & 9 & 0 & \textbf{9} & 0 & 0 & 0 & 0 & 0 & 0 & 2 \\
\midrule
Overall & 1889 & 4 & 1481 & 0 & 169 & 35 & 263 & 83 & 207 & 342 \\
\bottomrule
\end{tabular}
}
\vspace{0.2cm}
%
\caption{
Ablation study aggregated over large scenarios, showing absolute mean values and SEMs.
We study the effects of enabling/disabling rerouting ($r$), delaying ($d$), and restructuring ($s$).
}
\label{tab:ablation-aggregated}
\centering
\resizebox{\textwidth}{!}{%
\begin{tabular}{lrrrrrr}
\toprule
Option & Overload [\#] & Arrival Delay [h] & Sector-Number [\#] & Sector-Diff [\#] & Reroute [\#] & Reconfig [\#] \\
\midrule
r & $9696.23 \pm 424.57$ & $20410.69 \pm 662.55$ & $90925.69 \pm 1729.46$ & $159.90 \pm 6.00$ & $3250.78 \pm 66.75$ & $16016.11 \pm 810.76$ \\
$\neg r$ & $10526.43 \pm 432.70$ & $18834.70 \pm 741.47$ & $84098.67 \pm 1629.73$ & $144.91 \pm 5.36$ & $2929.04 \pm 66.10$ & $15651.06 \pm 760.66$ \\
\addlinespace[0.35em]
\midrule
\addlinespace[0.15em]
d & $8995.97 \pm 416.58$ & $32960.03 \pm 940.81$ & $94563.00 \pm 1806.79$ & $162.97 \pm 6.27$ & $3465.57 \pm 67.44$ & $17611.50 \pm 896.61$ \\
$d_p$ & $11226.69 \pm 440.04$ & $6285.36 \pm 225.54$ & $80461.35 \pm 1539.89$ & $141.84 \pm 5.04$ & $2714.25 \pm 65.13$ & $14055.68 \pm 656.57$ \\
\addlinespace[0.35em]
\midrule
\addlinespace[0.15em]
s & $5269.25 \pm 306.63$ & $5192.79 \pm 267.19$ & $85727.03 \pm 1647.63$ & $304.81 \pm 7.61$ & $5035.24 \pm 83.87$ & $31667.18 \pm 1077.64$ \\
$\neg s$ & $14953.42 \pm 516.28$ & $34052.61 \pm 924.97$ & $89297.32 \pm 1713.15$ & $0.00 \pm 0.00$ & $1144.58 \pm 25.92$ & $0.00 \pm 0.00$ \\
\bottomrule
\end{tabular}
}
\end{table}

\subsection{Experiment Scenarios and Setup}

We conducted benchmarks on three sets of instances: \emph{Small}~\cite{beiser_LPNMR_2026}, \emph{SOTA}~\cite{agustin_air_2012}, and \emph{Large}.
Table~\ref{tab:scenario-description} shows an overview of the considered instances.

\subparagraph{The Small scenario.}
Regarding traffic density, we vary the number of flights $|F| \in \{10, 20, \dots, 100\}$.
For each $|F|$, we evaluate three distinct random seeds.
En-route vertices possess a fixed capacity of $1$, while airport capacities are assumed to be unconstrained.
Time is partitioned into 1-hour timesteps ($T_{\textit{gran}} = 1$).

\subparagraph{SOTA Instances.}
The SOTA instances are standard benchmarks in the literature~\cite{berstimas_ATFM_2011,agustin_air_2012}.
Instances $c01, c02,$ and $c03$ are randomized scenarios designed to test comprehensive ATFCM features.
BLO instances are similarly randomized, where mBLO represents specific modifications (we do not possess mBLO2).
The $IC$ instance reflects an industry case, capturing a typical day of flight schedules and paths for a segment of the European ATM network.
Following related literature, time is partitioned into 15-minute intervals ($T_{\textit{gran}} = 4$).

\subparagraph{Large instances.}
The large scenario considers larger graphs with more flights per day.
The number of flights are $|F| \in \{1000,1778, 3162,5623, 10000,17782, 31622\}$ flights, where for each $|F|$ we consider three seeds ($11904657,150699,42$),
and capacity is relative to nominal capacity (optimal weather/staffing).
We scale atomic sector capacities uniformly from 100\% down to 10\% in 10\% decrements ($P \cdot 1.0 \rightarrow P \cdot 0.1$).
We set $T_{\textit{gran}} = 4$ (as for SOTA).

\subparagraph{Technical Setup.}
We implemented the heuristics in Python (3.13.7), and used Clingo (5.8.0) for ASP and Gurobi (13.0.0 build v13.0.0rc1) for MIP computation.
All experiments were conducted on the Copperhead CPU Cluster at TU Wien,
with 2x Intel Xeon Silver 4314 CPUs, 512GB RAM, and Ubuntu 22.04 (Kernel 5.15.0-131-generic).
We considered a TIMEOUT of $1800s$ and a MEMOUT of $35GB$.

\begin{figure}[t]
    \centering
      \includegraphics[width=14cm]{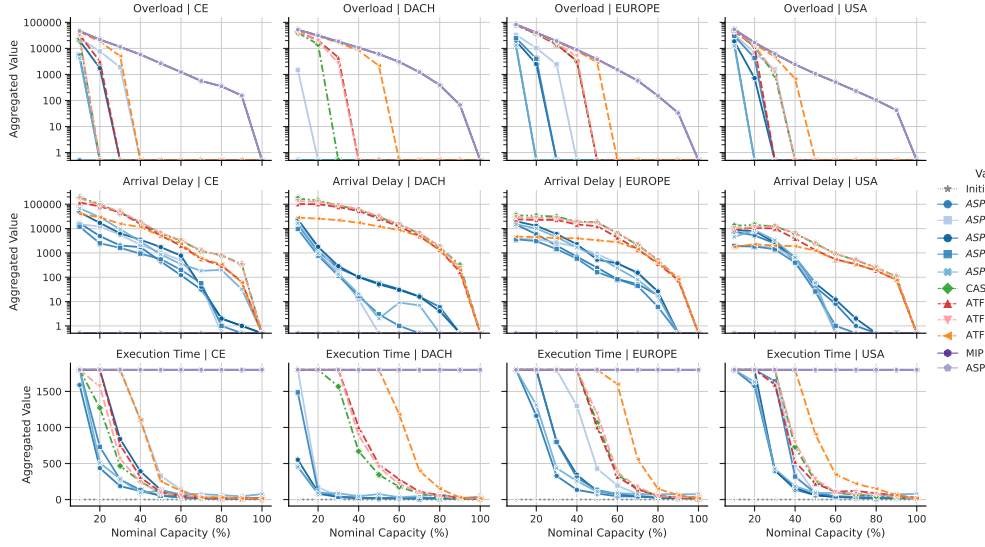}
    \caption{
    Nominal capacity scaling results on the four real-world topography large instances, comparing heuristic and exact methods.
    Exact methods run into time-/memouts in this setting.
    }
    \label{fig:pcap-scaling}
\end{figure}
\begin{figure}
    \centering
    \includegraphics[width=13cm]{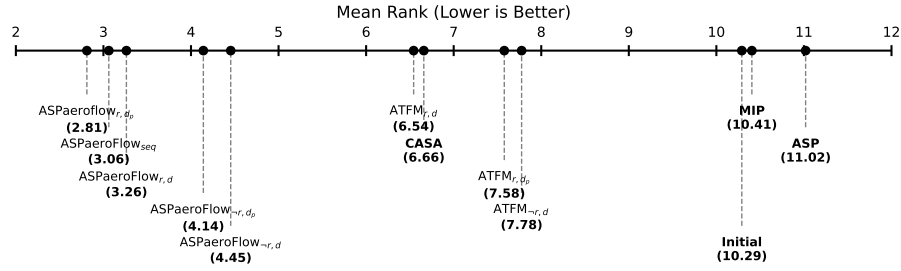}
    \caption{
    Significance graph showing the results of the post-hoc pair-wise Wilcoxon tests, where ranks indicate a hierarchy.
    All differences are statistically significant ($p < 0.05$).
    }
    \label{fig:significance-graph}
\end{figure}

\subsection{Results and Discussion}

\subparagraph{Heuristic vs. Exact Approaches.}
We show the number of solved instances ($\#S$) and tournament wins ($\#W$) in Table~\ref{tab:tournament-wins}.
While for the small experiments, the ASP-P variant achieves strong results ($\#W:83$), the $\text{ASPaeroFlow}_{r,d_p}$ variant achieves competitive results ($\#W:51$), while CASA wins only $\#W: 4$.
However, this is mostly due to the fact that even for small instances, the ASP-P approach runs into computational solving bottlenecks --- even on the small instances, ASP-P is not able to solve all instances in the time limit.
Comparing these methods on instances, which ASP-P solves to $0$ overload, $\text{ASPaeroFlow}_{r,d_p}$ achieves a lower arrival delay (an average value of $25.62 \pm 2.25$) compared to CASA (an average value of $64.83 \pm 10.19$).
MIP achieves worse results compared to ASP for solved small instances on arrival delay (an average value of $5.28 \pm 0.52$), but is overall able to solve slightly more instances ($\#S: 263$ vs. $\#S: 207$).
We emphasize that this should not be interpreted as ASP being in general a better solving technology than MIP; rather, the two approaches operate in fundamentally different feasible decision spaces. Specifically, we attribute the difference in results to ASP-P's ability to utilize DAC actions, where ASP-P uses on average $\ATFCMoSECCHANGES = 188.00 \pm 8.60$, compared to $0$ for MIP.
A similar behavior can also be seen when comparing the results of MIP and ASP-P for EA-3x3 and EUR-10x10.
EA-3x3 is a tiny graph with effectively no possibility for DAC ($|V| = 16$ with max $|S| = 3$), whereas for EUR-10x10 the possiblities for DAC are plenty ($|V|=117$ with max $|S|=10$).
As expected, MIP performs better on EA-3x3, whereas ASP-P performs better on EUR-10x10.
In the Appendix Table~\ref{tab:small-SOTA-exps} we show the detailed results for the small and SOTA scenarios, for 5 selected variants:
Baseline Initial, the best-performing heuristic variant for the small experiments $ASPaeroFlow_{r, d_p}$, CASA, and the exact methods MIP and ASP-P.
More detailed results are shown in the supplementary material.

\subparagraph{Scaling and Heuristic Approaches.}
In Figure~\ref{fig:pcap-scaling} we show the scaling results comparing nominal capacity.
The full scaling results comparing graph and flights size are available in the supplementary material.
For the navpoint and flights scaling results, we compare 4 heuristic approaches on $10\%$ nominal capacity (more in the supplementary material),
where $\text{ASPaeroFlow}_{r,d}$ and $\text{ASPaeroFlow}_{seq}$ are able to solve more instances with lower arrival delay than CASA/$\text{ATFM}_{r,d}$.
Interestingly, just comparing these results we can see that simultaneous optimization of $\text{ASPaeroFlow}_{r,d}$ achieves competitive results compared to $\text{ASPaeroFlow}_{seq}$,
as is able to solve more instances on sizes $|V| \in \{CE-G,USA-G,CE\}$, whereas $\text{ASPaeroFlow}_{seq}$ achieves better results on $|V| \in \{EUR-G,DACH\}$.
This behavior can also be seen in Figure~\ref{fig:pcap-scaling}, where nominal capacity alterations are shown for large graphs ($|V| \in \{\textit{CE}, \textit{DACH}, \textit{EUR}, \textit{USA}\}$):
the $\text{ASPaeroFlow}$ variants use approximately an equal execution time, whereas CASA/$\text{ATFM}$-variants are slower. 
Furthermore, we can observe that the results are instance-dependent and while $\text{ASPaeroFlow}$-variants are able to solve instances with nominal capacity $\geq 20\%$, the methods relying on ATFM-measures are only able to solve instances $\geq 40\%$.
On the SOTA scenarios, we show the strongest results in Table~\ref{tab:small-SOTA-exps}.
$\text{ASPaeroFlow}_{r,d_p}$, $\text{ASPaeroFlow}_{seq}$, $\text{ATFM}_{r,d_p}$, and $\text{ATFM}_{r,d}$ achieve the strongest results.
The strong results for the ATFM only methods can be explained by studying the differences between the instances.
While the generated instances BLOX, c0X, and mBLOX have little to no possibility for DAC, the industrial instance IC allows for plenty of reconfiguration.
Therefore, the results of BLOX, c0X, and mBLOX, are similar, where the ATFM methods have a slight edge as they are (implicitly) more likely to consider reroutes --- which also explains results of negative delay, which happen as for the SOTA instances the filed path is not necessarily the optimal one.
Comparisons to SOTA work besides the paper introducing ASP-P and MIP~\cite{beiser_LPNMR_2026} however, are hardly useful, due to differences in the model stemming from dynamic computation, (slightly) different definitions of demand/capacity measurement, objective functions, and time. 

\subparagraph{Statistical Significance.}
To evaluate the statistical significance of the algorithmic differences, we first performed a non-parametric Friedman test.
The test indicated a statistically significant difference in the lexicographic mean ranks across the 12 variants ($\chi^2_F(11, N=1889) = 15459.90$, $p < 0.001$).
Consequently, we conducted post-hoc pairwise Wilcoxon signed-rank tests, applying the Holm-Bonferroni correction to adjust for multiple comparisons.
The results demonstrate that all pairwise differences are statistically significant ($p < 0.001$), establishing a strict performance hierarchy of $\text{ASPaeroFlow}_{r,d_p}$, $\text{ASPaeroFlow}_{seq}$, $\text{ASPaeroFlow}_{r,d}$, followed by the others.
The final ordering of the variants, based on their mean ranks, is depicted in Figure~\ref{fig:significance-graph}.

\subparagraph{Sequential vs. Simultaneous Optimization.} 
We observe that simultaneous \textit{ASPaeroFlow} is highly competitive to $\textit{ASPaeroFlow}_{seq}$.
However, relying solely on aggregate statistical tests or scaling results across the entire benchmark yields inconclusive results:
while all simultaneous variants together solve and win more than $\textit{ASPaeroFlow}_{seq}$, no single simultaneous variant beats $\textit{ASPaeroFlow}_{seq}$ (Table~\ref{tab:tournament-wins}).
This occurs because sequential optimization possesses two inherent advantages: 
first, if an instance can be solved solely via DAC, sequential optimization achieves superior results since $\ATFCMoDELAY = 0$; 
second, the smaller search space per step permits more search steps compared to simultaneous optimization. 
Conversely, a major deficit exists: if pure DAC fails to eliminate overload, sequential optimization defaults to $\text{ATFM}_{r,d}$, which performs worse than simultaneous optimization.
We compare in more detail the variants $\text{ASPaeroFlow}_{r,d}$ and $\textit{ASPaeroFlow}_{seq}$:
evaluating all instances ($n = 1889$) on the two methods, $\text{ASPaeroFlow}_{r,d}$ achieves $\#W:864$ and $\#S: 1717$, while $\text{ASPaeroFlow}_{seq}$ achieves similar results ($\#W:843$, $\#S:1734$); with 182 draws. 
However, isolating the subset of instances where $\text{ASPaeroFlow}_{seq}$ fails to solve via DAC alone ($n = 1318$), $\text{ASPaeroFlow}_{r,d}$ demonstrates a clear advantage with $\#W:864$ ($\#S:1154$) compared to $\#W:441$ ($\#S:1171$) for $\text{ASPaeroFlow}_{seq}$; with $13$ draws.
This divergence is reflected in the statistical analysis.
Comparing only these two variants across the full dataset ($n = 1889$), the Wilcoxon signed-rank test yields mean ranks of $1.49$ for $\text{ASPaeroFlow}_{r,d}$ and $1.51$ for $\text{ASPaeroFlow}_{seq}$ ($p = 0.61$, not significant).
Note that the loss of significance compared to Figure~\ref{fig:significance-graph} stems exclusively from reducing the test space from $k=12$ to $k=2$; without the other variants skewing the global rank distribution, the high number of ties between these two highly competitive variants becomes statistically apparent. 
Interestingly, when applying the test to the dataset where $\text{ASPaeroFlow}_{seq}$ fails to solve via DAC alone ($n = 1318$), the mean ranks shift to $1.34$ for $\text{ASPaeroFlow}_{r,d}$ and $1.66$ for $\text{ASPaeroFlow}_{seq}$, yielding a statistically significant difference ($p < 0.001$).

\subparagraph{Ablation Study.}
Table~\ref{tab:ablation-aggregated} evaluates the impact of enabling rerouting ($r$), delaying ($d$), and restructuring ($s$).
Each feature reduces overloads, however for $r/\neg r$ the results are strictly speaking inconclusive due to overlapping SEMs.
Restructuring yields the largest reduction in both overload and arrival delay.
While the heuristic optimizes locally, a sector split at $t_1$ persists, providing global temporal benefits by resolving future overloads at $t_2 > t_1$.
Conversely, rerouting and delay effects remain temporarily localized.
For multi-leg flights, local delays may fail to resolve subsequent conflicts automatically and risk propagating downstream, inducing secondary overloads.

\section{Conclusion}
\label{sec:conclusion}
ASPaeroFlow demonstrates that bounded instance-space decomposition provides a computationally viable heuristic for joint ATFCM, scaling exact ASP to industry-sized instances.
Our ablation study reveals that DAC is the primary driver for overload reduction because sector partitioning provides global temporal benefits, whereas flow measures risk downstream propagation of overloads.
Furthermore, simultaneous optimization outperforms sequential approaches when DAC alone fails.
Future work will address the open challenges of stochastic disruptions and the integration of \emph{Explainable AI (XAI)} to extract rule-based justifications for automated scheduling decisions.

%
%
\bibliography{oasics-v2021-sample-article}


\appendix

\section{Appendix}

\subsection{AI Usage Statement}
To improve spelling and wording we used ChatGPT (various ChatGPT-5 versions) and Gemini 3.1 Pro.
Further, we used the same AI tools as coding assistants, for example to query software libraries, during debugging, and for documentation, however, substantial parts of the prototype were manually constructed (e.g., the entire ASP encoding).
Lastly, we used AI tools to help in the data analysis part to generate Python scripts which are able to generate \LaTeX~tables and matplotlib figures.

\subsection{Limitations.}
(1) \textit{Heuristic nature.} ASPaeroFlow is a heuristic without global optimality guarantees; solution quality can depend on decomposition bounds and the overload resolution order.
(2) \textit{Operational fidelity.} Reroutes and sector partitions are generated on a navpoint graph and may violate unmodeled constraints; we also omit tactical interventions and weather uncertainty. Time is discretized for discrete optimization, which may induce errors.
(3) \textit{Capacity and occupancy abstraction.} Capacity assumes partitioning enables parallel staffing.
(4) \textit{Network scope and objectives.} We partially model airport/terminal constraints, so mitigating en-route overload may shift congestion elsewhere; efficiency is proxied mainly by arrival delay and stability terms, without explicit fuel/emissions, heterogeneous airline costs, or stakeholder fairness constraints.

\subsection{Additional Experimental Results}

  \begin{figure}[t]
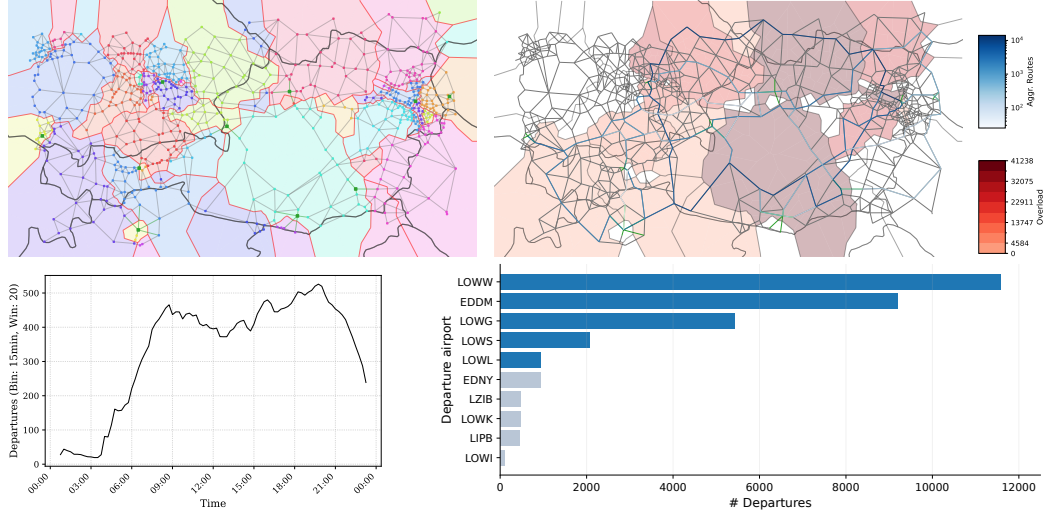

      \centering
      \begin{subfigure}[t]{0.45\textwidth}
          \includegraphics[width=0.99\textwidth]{imgs/04_initial_sectors_appendix/gabriel_08-0-CENTRAL-EUROPE-2019-06-01--2019-06-30-CAP-ENROUTE-1200-CLUSTERSIZE-30.pdf}
        \end{subfigure}
        \begin{subfigure}[t]{0.45\textwidth}
          \includegraphics[width=0.99\textwidth]{imgs/05_initial_overload_appendix/08-0-CENTRAL-EUROPE-2019-06-01--2019-06-30-CAP-ENROUTE-1200-CLUSTERSIZE-30.pdf}
        \end{subfigure}
        \begin{subfigure}[t]{0.08\textwidth}
          \includegraphics[width=0.8\textwidth]{imgs/05_initial_overload_appendix/08-0-CENTRAL-EUROPE-2019-06-01--2019-06-30-CAP-ENROUTE-1200-CLUSTERSIZE-30_legends.pdf}
        \end{subfigure}

        \begin{subfigure}[t]{0.4\textwidth}
            \includegraphics[width=5cm]{imgs/06_flight_counts/20260617_08.pdf}            
        \end{subfigure}
        \begin{subfigure}[t]{0.59\textwidth}
            \includegraphics[width=8cm]{imgs/07_airport_counts/08-0-CENTRAL-EUROPE-2019-06-01--2019-06-30-CAP-ENROUTE-1200-CLUSTERSIZE-30__DATA_S31622_42__departure_airports.pdf}  %
        \end{subfigure}
        
        \caption{
        Figure displaying an instance of CE (Central Europe), from the data generator.
        Left top: Initial navpoint-sector figure.
        Right top: Initial total overload exceedances ($K_o$) and aggregated trajectory-counts figure.
        Left bottom: Filed flights departure times.
        Right bottom: Histogram of departure airports.
        Instance: 31622 flights, CE.
        }
      \label{fig:schematic-general-graph}
  \end{figure}

\begin{table}[t]
\caption{
Small and SOTA scenarios (selected) results.
We show overload (O) and arrival delay (D) of the considered variants.
}
\label{tab:small-SOTA-exps}
\centering
\tiny
\resizebox{\linewidth}{!}{%
\setlength{\tabcolsep}{2.5pt} 
  \renewcommand{\arraystretch}{0.95} %
\begin{tabular}{lrrrrrrrrrrlrrrrrrrrrr}
\toprule
Instance & \multicolumn{10}{c}{\textbf{EA-3x3 (Seed 150699)}} & & \multicolumn{10}{c}{\textbf{CE-5x5 (Seed 150699)}} \\
\cmidrule(lr){2-11} \cmidrule(lr){13-22}
 & \multicolumn{2}{c}{Initial} & \multicolumn{2}{c}{$ASPaeroFlow_{r, d_p}$} & \multicolumn{2}{c}{CASA} & \multicolumn{2}{c}{MIP} & \multicolumn{2}{c}{ASP-P} & & \multicolumn{2}{c}{Initial} & \multicolumn{2}{c}{$ASPaeroFlow_{r, d_p}$} & \multicolumn{2}{c}{CASA} & \multicolumn{2}{c}{MIP} & \multicolumn{2}{c}{ASP-P} \\
\cmidrule(lr){2-3} \cmidrule(lr){4-5} \cmidrule(lr){6-7} \cmidrule(lr){8-9} \cmidrule(lr){10-11} \cmidrule(lr){13-14} \cmidrule(lr){15-16} \cmidrule(lr){17-18} \cmidrule(lr){19-20} \cmidrule(lr){21-22}
 & O & D & O & D & O & D & O & D & O & D & & O & D & O & D & O & D & O & D & O & D \\
\midrule
10 & 10 & 0 & 0 & 56 & 0 & 20 & \textbf{0} & \textbf{0} & 1 & 9 & & 1 & 0 & \textbf{0} & \textbf{0} & 0 & 2 & 0 & 2 & \textbf{0} & \textbf{0}  \\
20 & 22 & 0 & 0 & 311 & 0 & 156 & \textbf{0} & \textbf{15} & 7 & 33 & & 5 & 0 & 0 & 30 & 0 & 11 & 0 & 5 & \textbf{0} & \textbf{1}  \\
30 & 53 & 0 & 0 & 687 & 0 & 919 & \textbf{0} & \textbf{65} & 24 & 54 & & 13 & 0 & 0 & 102 & 0 & 79 & 0 & 11 & \textbf{0} & \textbf{4}  \\
40 & 80 & 0 & 0 & 1370 & 0 & 1247 & \textbf{0} & \textbf{230} & 54 & 102 & & 26 & 0 & 0 & 213 & 0 & 231 & 0 & 90 & \textbf{0} & \textbf{21}  \\
50 & 117 & 0 & \textbf{0} & \textbf{2115} & 0 & 2740 & 3 & 422 & 88 & 96 & & 37 & 0 & 0 & 342 & 0 & 506 & 0 & 187 & \textbf{0} & \textbf{25}  \\
60 & 133 & 0 & \textbf{0} & \textbf{3285} & 0 & 3351 & 12 & 482 & 95 & 127 & & 38 & 0 & 0 & 401 & 0 & 562 & 0 & 193 & \textbf{0} & \textbf{49}  \\
70 & 178 & 0 & \textbf{0} & \textbf{4504} & 0 & 5366 & 35 & 382 & 134 & 160 & & 59 & 0 & 0 & 748 & 0 & 903 & 0 & 305 & \textbf{0} & \textbf{57}  \\
80 & 204 & 0 & \textbf{0} & \textbf{6106} & 0 & 6302 & 59 & 250 & 157 & 178 & & 76 & 0 & \textbf{0} & \textbf{1005} & 0 & 1288 & 18 & 271 & 3 & 77  \\
90 & 241 & 0 & \textbf{0} & \textbf{7651} & 0 & 9517 & 88 & 186 & 192 & 230 & & 94 & 0 & \textbf{0} & \textbf{1356} & 0 & 2013 & 38 & 208 & 9 & 148  \\
100 & 287 & 0 & \textbf{0} & \textbf{9520} & 0 & 11034 & 120 & 111 & 241 & 199 & & 111 & 0 & \textbf{0} & \textbf{1578} & 0 & 2365 & 56 & 170 & 17 & 138  \\
\midrule
Instance & \multicolumn{10}{c}{\textbf{IND-4x10 (Seed 150699)}} & & \multicolumn{10}{c}{\textbf{USA-7x7 (Seed 150699)}} \\
\cmidrule(lr){2-11} \cmidrule(lr){13-22}
 & \multicolumn{2}{c}{Initial} & \multicolumn{2}{c}{$ASPaeroFlow_{r, d_p}$} & \multicolumn{2}{c}{CASA} & \multicolumn{2}{c}{MIP} & \multicolumn{2}{c}{ASP-P} & & \multicolumn{2}{c}{Initial} & \multicolumn{2}{c}{$ASPaeroFlow_{r, d_p}$} & \multicolumn{2}{c}{CASA} & \multicolumn{2}{c}{MIP} & \multicolumn{2}{c}{ASP-P} \\
\cmidrule(lr){2-3} \cmidrule(lr){4-5} \cmidrule(lr){6-7} \cmidrule(lr){8-9} \cmidrule(lr){10-11} \cmidrule(lr){13-14} \cmidrule(lr){15-16} \cmidrule(lr){17-18} \cmidrule(lr){19-20} \cmidrule(lr){21-22}
 & O & D & O & D & O & D & O & D & O & D & & O & D & O & D & O & D & O & D & O & D \\
\midrule
10 & 15 & 0 & 0 & 69 & 0 & 36 & 0 & 7 & \textbf{0} & \textbf{3} & & 1 & 0 & 0 & 10 & \textbf{0} & \textbf{1} & \textbf{0} & \textbf{1} & \textbf{0} & \textbf{1}  \\
20 & 21 & 0 & 0 & 192 & 0 & 152 & 0 & 34 & \textbf{0} & \textbf{5} & & 12 & 0 & 0 & 66 & 0 & 27 & 0 & 10 & \textbf{0} & \textbf{3}  \\
30 & 45 & 0 & 0 & 577 & 0 & 503 & 0 & 166 & \textbf{0} & \textbf{28} & & 26 & 0 & 0 & 312 & 0 & 251 & 0 & 63 & \textbf{0} & \textbf{7}  \\
40 & 70 & 0 & \textbf{0} & \textbf{904} & 0 & 1030 & 9 & 169 & 2 & 59 & & 49 & 0 & 0 & 689 & \textbf{0} & \textbf{468} & 9 & 117 & 4 & 27  \\
50 & 92 & 0 & \textbf{0} & \textbf{1395} & 0 & 1689 & 23 & 255 & 10 & 97 & & 61 & 0 & 0 & 1031 & \textbf{0} & \textbf{968} & 15 & 146 & 5 & 43  \\
60 & 136 & 0 & \textbf{0} & \textbf{2420} & 0 & 3115 & 77 & 133 & 38 & 125 & & 80 & 0 & \textbf{0} & \textbf{1368} & 0 & 1513 & 26 & 129 & 10 & 50  \\
70 & 135 & 0 & \textbf{0} & \textbf{3190} & 0 & 3241 & 72 & 177 & 35 & 147 & & 89 & 0 & \textbf{0} & \textbf{1721} & 0 & 1791 & 30 & 168 & 9 & 81  \\
80 & 186 & 0 & \textbf{0} & \textbf{4418} & 0 & 5481 & 131 & 85 & 54 & 177 & & 117 & 0 & \textbf{0} & \textbf{2472} & 0 & 2871 & 47 & 161 & 14 & 104  \\
90 & 193 & 0 & \textbf{0} & \textbf{5145} & 0 & 6016 & 138 & 111 & 77 & 186 & & 140 & 0 & \textbf{0} & \textbf{2988} & 0 & 3168 & 65 & 175 & 18 & 139  \\
100 & 233 & 0 & \textbf{0} & \textbf{6453} & 0 & 9523 & 181 & 87 & 100 & 197 & & 180 & 0 & \textbf{0} & \textbf{4083} & 0 & 5584 & 94 & 186 & 38 & 173  \\
\midrule
Instance & \multicolumn{10}{c}{\textbf{EUR-10x10 (Seed 150699)}} & Instance & \multicolumn{10}{c}{\textbf{SOTA}} \\
\cmidrule(lr){2-11} \cmidrule(lr){13-22}
 & \multicolumn{2}{c}{Initial} & \multicolumn{2}{c}{$ASPaeroFlow_{r, d_p}$} & \multicolumn{2}{c}{CASA} & \multicolumn{2}{c}{MIP} & \multicolumn{2}{c}{ASP-P} & & \multicolumn{2}{c}{Initial} & \multicolumn{2}{c}{$ASPaeroFlow_{r, d}$} & \multicolumn{2}{c}{$\textit{ASPaeroFlow}_{seq}$} & \multicolumn{2}{c}{$\textit{ATFM}_{r,d}$} & \multicolumn{2}{c}{$\textit{ATFM}_{r,d_p}$} \\
\cmidrule(lr){2-3} \cmidrule(lr){4-5} \cmidrule(lr){6-7} \cmidrule(lr){8-9} \cmidrule(lr){10-11} \cmidrule(lr){13-14} \cmidrule(lr){15-16} \cmidrule(lr){17-18} \cmidrule(lr){19-20} \cmidrule(lr){21-22}
 & O & D & O & D & O & D & O & D & O & D & & O & D & O & D & O & D & O & D & O & D \\
\midrule
10 & 8 & 0 & 0 & 2 & 0 & 15 & 0 & 4 & \textbf{0} & \textbf{0}  & BLO1 & 466 & 0 & 0 & 423 & 0 & 431 & \textbf{0} & \textbf{292} & 0 & 608 \\
20 & 14 & 0 & 0 & 87 & 0 & 69 & 0 & 11 & \textbf{0} & \textbf{3}  & BLO2 & 294 & 0 & 0 & 1897 & \textbf{0} & \textbf{1843} & 0 & 1897 & 0 & 1921 \\
30 & 25 & 0 & 0 & 93 & 0 & 128 & 0 & 33 & \textbf{0} & \textbf{1}  & BLO3 & 1884 & 0 & 0 & 27074 & 0 & 27074 & 0 & 27074 & \textbf{0} & \textbf{26845} \\
40 & 49 & 0 & 0 & 199 & 0 & 503 & 0 & 158 & \textbf{0} & \textbf{5}  & IC & 3368 & 0 & \textbf{0} & \textbf{85} & 0 & 222 & 0 & 26328 & 2744 & 6314 \\
50 & 62 & 0 & 0 & 305 & 0 & 1045 & 7 & 172 & \textbf{0} & \textbf{13}  & c01 & 466 & 0 & 0 & 423 & 0 & 431 & \textbf{0} & \textbf{292} & 0 & 608 \\
60 & 88 & 0 & 0 & 408 & 0 & 1621 & 23 & 208 & \textbf{0} & \textbf{41}  & c02 & 91 & 0 & \textbf{0} & \textbf{56} & 0 & 89 & \textbf{0} & \textbf{56} & 0 & 75 \\
70 & 110 & 0 & 0 & 621 & 0 & 2370 & 42 & 194 & \textbf{0} & \textbf{68}  & c03 & 171 & 0 & 0 & 83 & 0 & 39 & 0 & -86 & \textbf{0} & \textbf{-108} \\
80 & 148 & 0 & 0 & 943 & 0 & 3637 & 62 & 210 & \textbf{0} & \textbf{117}  & mBLO1 & 14 & 0 & 0 & -4 & 0 & 0 & \textbf{0} & \textbf{-18} & \textbf{0} & \textbf{-18} \\
90 & 156 & 0 & 0 & 884 & 0 & 4009 & 69 & 194 & \textbf{0} & \textbf{99}  & mBLO3 & 1884 & 0 & \textbf{0} & \textbf{27074} & \textbf{0} & \textbf{27074} & \textbf{0} & \textbf{27074} & 84 & 26409 \\
100 & 172 & 0 & \textbf{0} & \textbf{1181} & 0 & 5115 & 105 & 148 & 5 & 140 & &  &  &  &  &  &  &  &  &  &   \\
\bottomrule
\end{tabular}
}
\end{table}

In Table~\ref{tab:small-SOTA-exps} we show the detailed results for the small and SOTA scenarios.
Further, in Figure~\ref{fig:schematic-general-graph} we show an example of a generated instance produced by the data generator
and Figure~\ref{fig:significance-results} shows the detailed significance results.

\begin{figure}
    \centering
    \includegraphics[width=10cm]{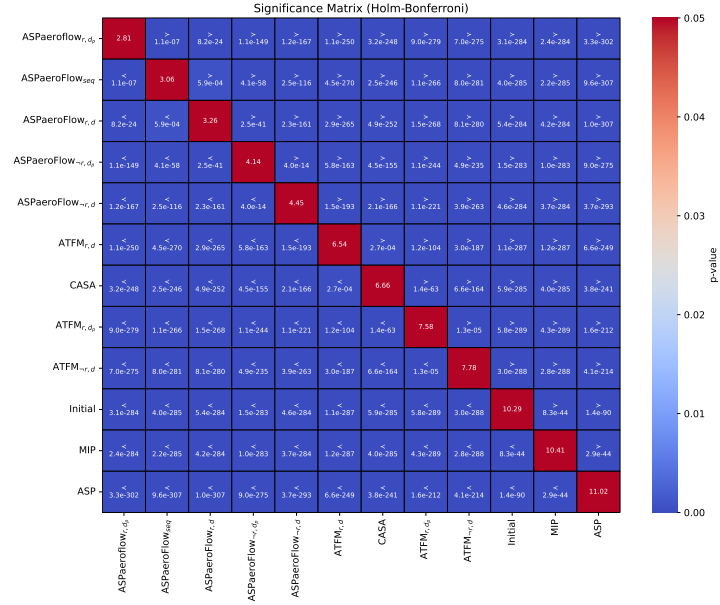}
    \caption{
    Matrix showing pairwise statistical significance results.
    }
    \label{fig:significance-results}
\end{figure}

\subsection{ASP Encoding Details}
\label{sec:appendix-asp}
\begin{tcolorbox}[colback=white, colframe=black!50, boxrule=0.5pt, arc=2pt, 
                  title=\textbf{Generated Facts: Flight Trajectories}, 
                  coltitle=black, colbacktitle=black!5,
                  boxsep=0pt, left=2pt, right=2pt, top=0pt, bottom=0pt,
                  toptitle=1pt, bottomtitle=1pt, before skip=7pt, after skip=0pt]
\small
\renewcommand{\arraystretch}{0.85}
\begin{tabularx}{\linewidth}{@{}l X@{}}
    $\mathtt{trj}(f, p)$ & Defines the available trajectories (route-delay options) $p$ for flight $f$. \\
    $\mathtt{fpl}(f, n, t)$ & Specifies the presence for a filed flight $f$ at navpoint $n$ at time $t$. \\
    $\mathtt{actd}(f, t, p)$ & Indicates the actual departure time $t$ for flight $f$ using path $p$. \\
    $\mathtt{plana}(f, t)$ & Defines the originally planned arrival time $t$ for flight $f$. \\
    $\mathtt{acta}(f, t, p)$ & Denotes the actual arrival time $t$ for flight $f$ given the chosen trj. $p$. \\
    $\mathtt{nxp}(\dots)$ & Encodes sequential trajectory segments: $(f, p, n_{i-1}, t_{i-1}, n_i, t_i)$. \\
    $\mathtt{chp}(\dots)$ & Propagates trajectory selection: $(f_{aff}, p) \leftarrow (f, p)$.
\end{tabularx}
\end{tcolorbox}
%
%
\begin{tcolorbox}[colback=white, colframe=black!50, boxrule=0.5pt, arc=2pt, 
                  title=\textbf{Generated Facts: Flight Trajectories}, 
                  coltitle=black, colbacktitle=black!5,
                  boxsep=0pt, left=2pt, right=2pt, top=0pt, bottom=0pt,
                  toptitle=1pt, bottomtitle=1pt, before skip=5pt, after skip=10pt]
\small
\begin{tabularx}{\textwidth}{@{}l X@{}}
    $\mathtt{config}(c, k_{conf})$ & Sector configuration $c$ associated with $k_{conf}$ underlying conflicts. \\
    $\mathtt{sec\_cap\_conf}(\dots)$ & Sets the capacity $k_{cap}$ of sector $s$ at time $t$ under configuration $c$---assuming all flights fixed except to-be-optimized flights. \\
    $\mathtt{nav\_assign\_conf}(\dots)$ & Maps navpoint $n$ to sector $s$ at time $t$ within configuration $c$.
\end{tabularx}
\end{tcolorbox}
\noindent
The following part of the encoding shows how guesses are performed.
Line~(1) shows how guesses for trajectories,
and Line~(2) shows guesses for sector configurations.
\begin{lstlisting}
1{chtrj(ID,P):trj(ID,P)}1 :- flightID(ID).
1{chc(CONFIG):config(CONFIG,_)}1.
\end{lstlisting}

\noindent
Next, we map the chosen trajectory to the navpoint trajectory that is actually flown, which is needed to compute the updated converted instance $\mathbf{M}'$ in Line~\ref{alg:heuristic-alg-1:line-7} in Algorithm~\ref{alg:heuristic-alg-1}.
Normal flight are shown in Lines~(1) and~(2), which map for each flown edge $(v_0,v_1)$ the navpoints to the trajectory.
Line~(3) is a special case for handling atomic flights.
\begin{lstlisting}
nvpt_f(ID,NAV0,T) :- actd(ID,TSTART,P),chtrj(ID,P),nxp(ID,P,NAV0,T0,_,_),T=T0+TSTART.
nvpt_f(ID,NAV1,T) :- actd(ID,TSTART,P),chtrj(ID,P),nxp(ID,P,_,_,NAV1,T1),T=T1+TSTART.
nvpt_f(ID,NAV0,T) :- actd(ID,TSTART,P),chtrj(ID,P),single_pos(ID,P,NAV0,T0),T=T0+TSTART.
\end{lstlisting}

\noindent
Line~(1) maps a chosen sector configuration to the effects it has on the navpoint sector assignments
and Line~(2) maps the capacities of the chosen sector configuration to the $sec\_cap$ predicate.
\begin{lstlisting}
nav_ass(NAV,SEC,TIME) :- chc(CONFIG), nav_assign_conf(NAV,SEC,TIME,CONFIG).
sec_cap(SEC,CAPACITY,TIME) :- chc(CONFIG), sec_cap_conf(SEC,CAPACITY,TIME,CONFIG).
\end{lstlisting}

\noindent
Combining the chosen trajectory and the chosen config, we are left to map the flown navpoint trajectory to the actual effects on the sectors.
Lines~(1) and~(2) handle the standard behavior for a trajectory; they encode the standard behavior for a flight $f \in \mathbf{F}$ reaching navpoint $v_a$ at $t_a$, followed by arriving at $v_b$ at $t_b$.
Then $f$ is in sector $\textit{sec}^{-1}(v_a,t)$ in $t \in [t_a, t_a + \lfloor \frac{\Delta t}{2} \rfloor]$ and in sector $\textit{sec}^{-1}(v_b,t')$ for $t' \in [t_a + \lfloor \frac{\Delta t}{2} \rfloor + 1, t_b]$.
Additionally, we consider two special cases: first atomic flights (Line~(3)) and then flights staying at a navpoint (Line~(4)).
\begin{lstlisting}
sec_f(ID,SEC0,T) :- actd(ID,TSTART,P),chtrj(ID,P),time(T),nxp(ID,P,NAV0,T0,NAV1,T1),NAV0!=NAV1,T>=T0+TSTART,T<=T1+TSTART,DT=(T1-T0)/2,T<=DT+T0+TSTART,nav_ass(NAV0,SEC0,T).
sec_f(ID,SEC1,T) :- actd(ID,TSTART,P),chtrj(ID,P),time(T),nxp(ID,P,NAV0,T0,NAV1,T1),NAV0!=NAV1,T>=T0+TSTART,T<=T1+TSTART,DT=(T1-T0)/2,T>DT+T0+TSTART,nav_ass(NAV1,SEC1,T).
sec_f(ID,SEC0,T) :- actd(ID,TSTART,P),chtrj(ID,P),single_pos(ID,P,NAV0,T0),nav_ass(NAV0,SEC0,T),T=T0+TSTART.
sec_f(ID,SEC0,T) :- actd(ID,TSTART,P), chtrj(ID,P), time(T), nxp(ID,P,NAV0,T0,NAV0,T1), T >= T0 + TSTART, T <= T1 + TSTART, nav_ass(NAV0,SEC0,T).
\end{lstlisting}

\noindent
If an overload occurs, we store this in the predicate $overload$ (Line~(1)),
an arrival delay is stored in $arrival\_delay$ (Line~(2)),
and if a flight is rerouted, it is stored in $reroute$ (Line~(3)).
\begin{lstlisting}
overload(SEC, T, LOAD-CAPACITY) :- sec_cap(SEC,CAPACITY,T), #count{ID:sec_f(ID,SEC,T)} = LOAD, LOAD > CAPACITY.
arrival_delay(ID,Y-T) :- chtrj(ID,P), plana(ID,T), acta(ID,Y,P).
reroute(ID) :- flightID(ID), not chtrj(ID,0).
\end{lstlisting}
\clearpage

\section{Supplementary Material}

\subsection{Model Summary and Decomposition Bounds}

See Table~\ref{tab:model-overview} for a quick overview of the model.

\begin{table}[t]
\caption{Compact summary of the model.}
\label{tab:model-overview}
\centering
\tiny
\setlength{\tabcolsep}{5pt}
\renewcommand{\arraystretch}{1.00}

\begin{tabular}{L@{\hspace{-10pt}}R}
\toprule
\multicolumn{2}{l}{\textbf{Model overview}}\\
\midrule
\textbf{Instance input} &
$\mathcal{I}=(G,T_{\textit{gran}},T,\textit{sec},\textit{cap},F,A)$ navpoint graph $G=(V,E)$, time granularity $T_{\textit{gran}}$, set of timesteps $T$, initial sector config. \textit{sec}, atomic capacities \textit{cap}, filed flights $F$, and aircraft flight mappings $A$.\\

\textbf{Solution} &
$\mathbf{S}=(\mathbf{T},\mathbf{SEC},\mathbf{F})$ with a potentially extended time horizon $|\mathbf{T}| \geq |T|$, a pot. changed sector config. $\mathbf{SEC}: V \times\mathbf{T}\rightarrow 2^V$, and with pot. adjusted trajectories $\mathbf{F}$. \\

\textbf{Actions} &
Delay and reroute trajectories via $\mathbf{A}$; DAC via $\mathbf{SEC}$.
\\

\textbf{Hard constraints} &
(1) Flight feasibility: each realized flight matches a filed one \;
(2) Safety/capacity: $\forall i,t:\ q(i,t)=\sum_{f}\textit{over}(f,i,t)\le \textit{cap}(i,t)$.\\

\textbf{Soft constraints} &
Lexicographic minimization:
total arrival delay \ATFCMoDELAY,
\# active total sectors \ATFCMoNUMSEC,
\# sector changes \ATFCMoSECCHANGES,
\# regulated flights \ATFCMoREROUTED,
\# reconfigurations \ATFCMoRECONFIG\\
\bottomrule
\end{tabular}
\end{table}

\subsection{Auxiliary Predicates}

Here we define the auxiliary predicates.
Line~(1) collects flight IDs,
Lines~(2) and~(3) collect the minimum and maximum time respectively,
Line~(4) provide dense times (no holes) between min and max time, and 
Lines~(5) and~(6) ensure that all flights occur.
\begin{lstlisting}
flightID(ID) :- fpl(ID,_,_).
time_minimum(T) :- T = #min{T':nav_assign_conf(_,_,T',_);T':nxp(_,_,_,T',_,_);T':single_pos(_,_,_,T')}.
time_maximum(T) :- T = #max{T':nav_assign_conf(_,_,T',_);T':nxp(_,_,_,_,_,T');T':single_pos(_,_,_,T')}.
time(TMIN..TMAX) :- time_minimum(TMIN), time_maximum(TMAX).
flight_occurs(ID) :- sec_f(ID,_,_).
:- flightID(ID), not flight_occurs(ID).
\end{lstlisting}

\subsection{Data Generator Details}
\label{sec:appendix-datagen}

The scarcity of open data in Air Traffic Flow and Capacity Management (ATFCM) remains an \emph{open challenge}.
To address this and support initiatives like the Open Science Alliance for ATM, we extend the open-source \emph{ASPaeroFlow-DataGenerator}\footnote{GitHub: \url{https://github.com/alexl4123/ASPaeroFlow-DataGenerator}}.
Our enhancement supplements the previous grid-graph model with realistic navpoint generation.

\subparagraph{Navpoint Data Generation.}
The generator operates in two modes:
\begin{itemize}
    \item \textbf{Real-World Topology:} Real-world navpoints and airports (our contribution).
    \item \textbf{Hybrid Topology:} Grid-graph navpoints with real-world airports.
\end{itemize}
Geographic data merges \texttt{OurAirports-Data} (Unlicense) and \texttt{BlueSky} (GNU GPL), yielding a graph of ~160,000 navpoints (world-wide).
The generator implements spatial filters (minimum separation), altitude specifications, and topological structures like \emph{Gabriel graphs} and \emph{local cliques}. 
Sectors are generated by a BFS algorithm, which ensures connected sectors.

\subparagraph{Flight Data Generation.}
The generator produces industry-sized synthetic instances using real geographical data~\cite{beiser_LPNMR_2026}.
Flight data is generated using a probabilistic model derived from the \textit{OpenSky Network} COVID-19 dataset (specifically using data from June 2019, representing pre-pandemic traffic levels).
A \emph{Poisson model} defines airport sending strength.
Subsequently, pairs are generated by sampling from a discrete probability distribution, which are transformed into trajectories by a shortest-path trajectory generator.
Multi-leg flights are generated based on aircraft~availability.
Formal statistical validation of the generated data against historical flight distributions is planned for future work.

\textbf{Details}.
Let $A = \textit{airport(G)}$ be the airports, then each airport $a \in A$ has for a timestep $t_b \in T$ a sending strength $\lambda_{a,t_b}$.
Let $d$ be a day of a number of specified days $d \in D$, then the actual number of departures is $\#\textit{departures}(a,t,d)$, and the actual number of departures at $a$ with destination $a_d \in A$ is $\#\textit{dest}(a,t,d,a_d)$.
The probabilistic model is created by aggregating the departures into a matrix $\Lambda \in \mathbb{N}^{|A| \times |T|}$, where $\lambda_{a,t} = \Lambda_{[a,t]} = \frac{\sum_{d \in D} \#\textit{departures}(a,t,d)}{|D|}$.
For data generation, we sample the number of departures for an airport $a_o \in A$, and for a timestep $t\in T$, from $\#a_o \sim \text{Pois}(\lambda_{a_o,t})$.
It remains to generate for each departure at airport $a_o$, a destination airport $a_d$:
Let $\Psi \in \mathbb{N}^{|A| \times |A| \times |T|}$,
where $\psi_{a_o,a_d,t} = \Psi_{[a_o,a_d,t]} = \frac{\sum_{d \in D} \#\textit{dest}(a_o,t,d,a_d)}{|D|}$.
Further, let $a_o \in A$ and $t \in T$, then the probability for choosing destination $a_d \in A$ is $P(a_d \mid a_o,t) = \frac{\psi_{a_o,a_d,t} + \alpha}{\sum_{a \in A} \psi_{a_o,a_d,t} + \alpha}$,
where $\alpha \in \mathbb{R}$ is a small constant.
Destinations are drawn as a categorical random variable:
$a_d \sim \textit{Categorical}\left( P(a_d \mid a_o,t)_{a_d \in A} \right)$.

\subsection{A Brief intro to Answer Set Programming}

ASP~\cite{gelfond_logic_2002} is a declarative, logic-based, approach towards problem solving.
ASP originates from logic programming. It has evolved from a formal logical framework into a practically mature approach with a rich ecosystem, with solvers like Clingo~\cite{gebser_theory_2016}.
Compared to other similar approaches for problem solving, ASP is often considered as more intelligible, which makes it a good candidate for safety-critical systems.
This combination has led to a more widespread usage of ASP also outside of its academic origin~\cite{falkner_industrial_2018}.
%
We present necessary notions for understanding our encoding; for details we refer to the standard literature~\cite{eiter_answer_2009,schaub_special_2018}.
An (non-ground normal) ASP program $\Pi$ consists of rules $r \in \Pi$ of the form
$p_1(\mathbf{X}_1) \leftarrow p_{2}(\mathbf{X}_{2}), \ldots, p_m(\mathbf{X}_m), \neg p_{m+1}(\mathbf{X}_{m+1}), \ldots, \neg p_n(\mathbf{X}_n)$,
where $1 \leq m \leq n$.
Each $p_i$ is a predicate with an arity $|\mathbf{X}_i|$ and has terms $\mathbf{X}_i$, where $x \in \mathbf{X}_i$ can be a constant or a variable.
We provide a series of examples for introducing ASP.
\begin{example}
\label{ex:basic-rule}
Rules are read right to left.
The following is a non-ground rule with variable $ID$ and constant $0$, that describes that a flight is rerouted, unless its chosen path is $0$.
\begin{lstlisting}
reroute(ID) :- flightID(ID), not chosen_path(ID,0).
\end{lstlisting}
\end{example}

\noindent
We say head to $H_r \coloneqq \{a_1, \ldots, a_l\}$,
positive body to $B_r^+ \coloneqq \{a_{l+1}, \ldots, a_m\}$,
negative body to $B_r^{-} \coloneqq \{a_{m+1}, \ldots, a_n\}$, and
body to $B_r \coloneqq B_r^+ \cup B_r^{-}$.
We say a rule $r \in P$ is normal iff $|H_r| \leq 1$,
a constraint iff $|H_r| = 0$,
disjunctive iff $|H_r| > 1$, and
positive iff $|B_r^{-}| = 0$.
\begin{example}
Constraints are used to restrict the solution space.
The following rule says it must not be the case that a flight does not occur:
\begin{lstlisting}
:- flightID(ID), not flight_occurs(ID).
\end{lstlisting}
\end{example}

\noindent
We use aggregates (count, min, and max), choice-rules, and soft-constraints as usual~\cite{gebser_clingo_2019}.
\begin{example}
The following rule collects information about overloaded sectors. 
We use the $\#count$ aggregate to count the demand.
\begin{lstlisting}
overload(SEC, T, LOAD-CAPACITY) :- sec_cap(SEC,CAPACITY,T), #count{ID:sector_flight(ID,SEC,T)} = LOAD, LOAD > CAPACITY.
\end{lstlisting}

\noindent
Soft constraints are used to specify objective functions. The following minimizes the overload.
\begin{lstlisting}
:~ overload(X,T,OVER). [OVER@10,X,T]
\end{lstlisting}
\end{example}

\noindent
Choice-rules let one introduce ``guesses'' --- thereby, are essential for defining the search space.
The following choice-rule encodes that per flight exactly one path must be chosen.
\begin{lstlisting}
1{chosen_path(ID,P):paths(ID,P)}1 :- flightID(ID).
\end{lstlisting}

Solving an ASP program is usually done in a two-step process according to the ground-and-solve paradigm.
Grounding a program $\Pi$, denoted as $\mathcal{G}(\Pi)$, refers to the instantiation of the variables by their domain values.
Let $HB(\mathcal{G}(\Pi)$ be the Herbrand Base, so the set of all possible atoms.
Given a ground program we can compute an answerset (a solution):
An interpretation $I$ is a set of atoms $I \subseteq HB(\mathcal{G}(\Pi))$.
$I$ \emph{satisfies} a rule $r \in \mathcal{G}(\Pi)$ iff $(H_r \cup B_r^{-}) \cap I \not = \emptyset$ or $I \not \models B_r^+$, so either for an atom $a \in B_r^+$ $a \not \in I$.
$I$ is a \emph{model} of $\mathcal{G}(\Pi)$ iff it satisfies all rules of $\mathcal{G}(\Pi)$.
The \emph{Gelfond-Lifschitz (GL) reduct} of~$\mathcal{G}(\Pi)$ under~$I$ is the program~$\mathcal{G}(\Pi)^I$ obtained
from $\mathcal{G}(\Pi)$ by first removing all rules~$r$ with
$B^-_r{\,\cap\,} I\neq \emptyset$ and then removing all~$\neg a$ where
$a \in B^-_r$ from the remaining
rules~$r$ \cite{gelfond_classical_1991}. 
$I$ is an \emph{answer set} of a program~$\mathcal{G}(\Pi)$ iff $I$ is a \emph{
minimal model} (w.r.t.~$\subseteq$) of~$\mathcal{G}(\Pi)^I$.

\begin{example}
Let $\Pi$ be the rule of Example~\ref{ex:basic-rule} with the following facts $\mathcal{F}$
\begin{lstlisting}
flightID(0). flightID(1). chosen_path(0,0). chosen_path(1,1).
\end{lstlisting}
Then $\mathcal{I}_0 = \mathcal{F} \cup \{\textit{reroute}(1)\}$ is an answerset,
$\mathcal{I}_1 = \mathcal{F} \cup  \{\textit{reroute}(1), \textit{reroute}(0)\}$ is a model of $\mathcal{G}(\Pi)$ (but not an answerset),
and $\mathcal{I}_2 = \mathcal{F}$ does not satisfy $\mathcal{G}(\Pi)$. 
\end{example}

\subsection{Full Model (no abbreviations)}

\begin{lstlisting}
% Encoding in the ASPaeroFlow Optimizer:
% GET ALL FLIGHT IDS:

flightID(ID) :- flightPlan(ID,_,_).

all_flightID(ID) :- flightID(ID).
all_flightID(ID) :- sector_flight(ID,_,_).
all_flightID(ID) :- navpoint_flight(ID,_,_).

% GUESSES:
1{chosen_config(CONFIG):config(CONFIG,_)}1.
1{chosen_path(ID,P):paths(ID,P)}1 :- flightID(ID).

reroute(ID) :- flightID(ID), not chosen_path(ID,0).

% GET CHOSEN SECTORS:
navpoint_sector_assignment(NAV,SEC,TIME) :- chosen_config(CONFIG), possible_assignment(NAV,SEC,TIME,CONFIG).
sector_capacity(SEC,CAPACITY,TIME) :- chosen_config(CONFIG), possible_sector_capacity(SEC,CAPACITY,TIME,CONFIG).


time_minimum(T) :- T = #min{T':possible_assignment(_,_,T',_);T':next_pos(_,_,_,T',_,_);T':single_pos(_,_,_,T')}.
time_maximum(T) :- T = #max{T':possible_assignment(_,_,T',_);T':next_pos(_,_,_,_,_,T');T':single_pos(_,_,_,T')}.
time(TMIN..TMAX) :- time_minimum(TMIN), time_maximum(TMAX).

% FOR STARTING AIRPORT:
sector_flight(ID,SEC0,T) :- actual_flight_operations_start_time(ID,TSTART,P), chosen_path(ID,P), single_pos(ID,P,NAV0,T0), navpoint_sector_assignment(NAV0,SEC0,T), T = T0 + TSTART.
sector_flight(ID,SEC0,T) :- actual_flight_operations_start_time(ID,TSTART,P), chosen_path(ID,P), time(T), next_pos(ID,P,NAV0,T0,NAV0,T1), T >= T0 + TSTART, T <= T1 + TSTART, navpoint_sector_assignment(NAV0,SEC0,T).
% EN-ROUTE:
sector_flight(ID,SEC0,T) :- actual_flight_operations_start_time(ID,TSTART,P), chosen_path(ID,P), time(T), next_pos(ID,P,NAV0,T0,NAV1,T1), NAV0 != NAV1, T >= T0 + TSTART, T <= T1 + TSTART, DT = (T1 - T0)/2, T <= DT + T0 + TSTART, navpoint_sector_assignment(NAV0,SEC0,T).
sector_flight(ID,SEC1,T) :- actual_flight_operations_start_time(ID,TSTART,P), chosen_path(ID,P), time(T), next_pos(ID,P,NAV0,T0,NAV1,T1), NAV0 != NAV1, T >= T0 + TSTART, T <= T1 + TSTART, DT = (T1 - T0)/2, T > DT + T0 + TSTART, navpoint_sector_assignment(NAV1,SEC1,T).

navpoint_flight(ID,NAV0,T) :- actual_flight_operations_start_time(ID,TSTART,P), chosen_path(ID,P), single_pos(ID,P,NAV0,T0), T = T0 + TSTART.
navpoint_flight(ID,NAV0,T) :- actual_flight_operations_start_time(ID,TSTART,P), chosen_path(ID,P), next_pos(ID,P,NAV0,T0,_,_), T = T0 + TSTART.
navpoint_flight(ID,NAV1,T) :- actual_flight_operations_start_time(ID,TSTART,P), chosen_path(ID,P), next_pos(ID,P,_,_,NAV1,T1), T = T1 + TSTART.

% CAPACITY CONSTRAINT:
overload(SEC, T, LOAD-CAPACITY) :- sector_capacity(SEC,CAPACITY,T), #count{ID:sector_flight(ID,SEC,T)} = LOAD, LOAD > CAPACITY.
% WITH DIFFS:
%overload(SEC1,T,LOAD-CAPACITY) :- sector_capacity(SEC1,CAPACITY,T), #count{ID:sector_flight(ID,SEC1,T),sector_flight(ID,SEC2,T-1), SEC1!=SEC2} = LOAD, LOAD > CAPACITY.

%%%%%%%%%%%%%%%%%%%%%%%%%%%%%%%%%%%%%%%%%%%%%%%%%%%%%%%%%%%%%%%

% FLIGHT MUST OCCUR:
flight_occurs(ID) :- sector_flight(ID,_,_).
:- flightID(ID), not flight_occurs(ID).


arrival_delay(ID,Y-T) :- chosen_path(ID,P), planned_arrival_time(ID,T), actual_arrival_time(ID,Y,P).

% ------------- OPTIMIZATION ----------------
% ------------------------------------------------------------------------------
% PRIMARY WEAK CONSTRAINT:
:~ overload(X,T,OVER). [OVER@10,X,T]
:~ chosen_config(CONFIG), config(CONFIG,TOTAL_OVER). [TOTAL_OVER@10,CONFIG]
% SECONDARY WEAK CONSTRAINT:
:~ arrival_delay(ID,DIFF). [DIFF@9,ID]
% TERTIARY WEAK CONSTRAINT:
:~ chosen_config(CONFIG), config_number_sectors(CONFIG,NUMBER). [NUMBER@8,CONFIG]
% QUINARY WEAK CONSTRAINT:
:~ reroute(ID). [1@7,ID]
% QUARTERNARY WEAK CONSTRAINT:
:~ chosen_config(CONFIG). [CONFIG@6,CONFIG]
\end{lstlisting}

\subsection{Dynamic MIP Model}
We simulate the SOTA MIP model with the following dynamic MIP model.
Dynamic in the sense that the ground delay and reroute alternatives are generated by the Python wrapper.
Otherwise, we restrict the MIP model to a SOTA model in the sense that it only optimizes routes and delay, but not DAC.

\begin{figure*}
\begin{flalign}
    \notag &\textbf{Simulated SOTA MIP Model}&& \\
    &\label{eq:mip-real-01} \min \sum_{f \in F} \sum_{p \in P_f} d_p \cdot \omega_{\textit{dest}_f, t_{\textit{dest}}}^{f,p} \\
    &\nonumber \textbf{Subject to:}\\
    &\label{eq:mip-real-03} \omega_{j,t}^{f,p} \leq \omega_{i,t+1}^{f,p} && \textit{For all } f \in F, p \in P_f, t \in T^{f,p}, t+1 \in T^{f,p}, j = S_t^{f,p}, i = S_{t+1}^{f,p}\\
    &\label{eq:mip-real-04} \sum_{p \in P_t^f, j = S_t^{f,p}} \omega_{j,t}^{f,p} \leq 1 && \textit{For all } f \in F, t \in T_f\\
    &\label{eq:mip-real-05} \sum_{p \in P_{j,t}^f} \omega_{j,t}^{f,p} \leq 1 && \textit{For all } f \in F, t \in T_f, j = S_t^f\\
    %
    %
    &\label{eq:mip-real-07} 1 \leq \sum_{p \in P_f} \omega_{s,t_{s,\textit{min}}}^{f,p} && \textit{For all } f \in F, s = \textit{source}_f\\
    &\label{eq:mip-real-07-a} \sum_{p \in P_f} \omega_{s,t_{s,\textit{min}}}^{f,p} \leq 1 && \textit{For all } f \in F, s = \textit{source}_f\\
    &\label{eq:mip-real-08} \omega_{j,t}^{f,p} \in \{0,1\} && \textit{For all } f \in F, p \in P_f, t \in T^{f,p}, j = S^{f,p}_t\\
    \nonumber &\textbf{Consecutive flights}\\
    &\label{eq:mip-real-09} (1 - \omega_{i,t_{dest}}^{f_0,p'}) \geq \omega_{j,t_{start}}^{f_1,p''} && \textit{For all } f_0, f_1 \textit{ s.t. } f_0,f_1 \in A \textit{ and}\\
        \nonumber &&& \textit{planned landing time of } f_0 \textit{ is before planned departure time of } f_1\\
        \nonumber &&& \textit{but path one is } t_{dest} \geq t_{start}\\
    \nonumber &\textbf{Capacity Slots}\\
    &\label{eq:mip-real-02} \sum_{f \in F} \sum_{p \in P_f } \omega_{j,t}^{f,p} < C_j(t) &&\textit{For all } j \in S, t \in T
    \end{flalign}
\vspace{-0.7cm}
\end{figure*}

Let $F$ be the set of flights, where $f \in F$ has a departure ($\textit{source}_f$) and an arrival ($\textit{destination}_f$) airport,
with associated departure time $t_{\textit{start}}$ and expected arrival time $\hat{t}_{\textit{dest}}$.
Let $\mathcal{G}$ be the navpoint graph.
The set of paths $P$ is the set of simple paths between any two vertices in the navpoint graWe create for each jobph $\mathcal{G}$.
Then let $P_f$ be the set of paths restricted to f: $P_f = \{p \mid p = \langle \textit{source}_f, \ldots, \textit{destination}_f \rangle \in P \}$.

Further, let $T$ be the set of times in the whole application,
whereas, $T^{f,p}$ is the set of times that an airplane needs to travel from $\textit{source}_f$ to $\textit{destination}_f$.
Given a flight $f \in F$, a path $p \in P_f$, and a specific time during the flight $t \in T^{f,p}$,
we can compute the estimated current sector-position of the flight $j = S^{f,p}_t$.
Lastly, let $S$ be the set of sectors.
We interpret the sectors as the \emph{Voronoi} diagram of the navpoints, where each sector may have a set of corresponding navpoints.
The location of the flight along a path is computed according to the Voronoi interpretation.
A sector $j \in S$ has a capacity per time unit $C_j(t)$.
Example:  
When a flight $f$ starts at $t=0$ at navpoint $a$ (sector $A$) and travels to navpoint $b$ (sector $B$) - and the duration of the flight takes $\Delta t = 3$ units of time,
then flight $f$ is in sector A during time units $\{0,1\}$,
and in sector B during time units $\{2,3\}$.

Further, let $\textit{isFirst}(j,t,f,p)$ evaluate to true iff $\forall t' \in T^{f,p}: t'<t$ it holds that $i=S^{f,p}_{t'}$, where $i \not = j$.
Conversely, $\textit{isLast}(j,t,f,p)$ evaluates to true iff $\forall t' \in T^{f,p}: t' > t$ it holds that $i = S^{f,p}_{t'}$, where $i \not j$.

We treat capacity violations as a hard constraint and optimize for efficiency (to be more in line with SOTA work).
In the future, we plan to adapt this SOAT simulation model to a new model in our formalism, by incorporating DAC optimization and lexicographic optimization.

\emph{Description of Equations}.
We define variables for flightss, delays, and paths in Equation~(\ref{eq:mip-real-08}).
We require that exactly one departing flight is chosen (Equations~(\ref{eq:mip-real-07}) and~(\ref{eq:mip-real-07-a})).
If a path is chosen it must be kept throughout the flight (Equation~(\ref{eq:mip-real-03})) and there must not be more than one path at any other node (Equation~(\ref{eq:mip-real-04})).
Similarly, at any node along the path there must only be one path active (Equation~(\ref{eq:mip-real-05})).
At any given timestep, a sector must not be overloaden (Equation~(\ref{eq:mip-real-02})) and for an airplane with multiple flights per day, it must not be that the departure time $t_{\textit{start}}$ of the next flight is before the landing time of the previous flight $t_{\textit{dest}}$ (Equation~(\ref{eq:mip-real-09})).
One may switch Equation~(\ref{eq:mip-real-02}) to Equation~(\ref{eq:mip-real-02-diff}) to switch demand measurement from number of airplanes in a sector, to airplanes entering a sector.

\begin{flalign}
    &\nonumber \textbf{Diff Capacity}\\
    &\label{eq:mip-real-02-diff} \sum_{f \in F} \sum_{p \in P_f \land \textit{isFirst}(j,t,f,p)} \omega_{j,t}^{f,p} < C_j(t) &&\textit{For all } j \in S, t \in T
\end{flalign}

\subsection{Additional Experimental Details}
We provide additional experimental details, not shown in the main part or in the Appendix.
In Figure~\ref{fig:nominal-grid-graphs} we show the results for the nominal capacity scaling results and in Figure~\ref{fig:significance-results} we show the significance matrix, showing statistical significance between different variants.
Tables~\ref{tab:summary-by-nominal-capacity}--\ref{tab:summary-by-scenario} show aggregated results for the larger scenarios,
by highlighting the capacity, instances size, and scenario dependence, respectively.
Detailed results can be found in the supplementary material.
Table~\ref{tab:summary-by-variant-new} shows aggregated results.

\begin{figure}[t]
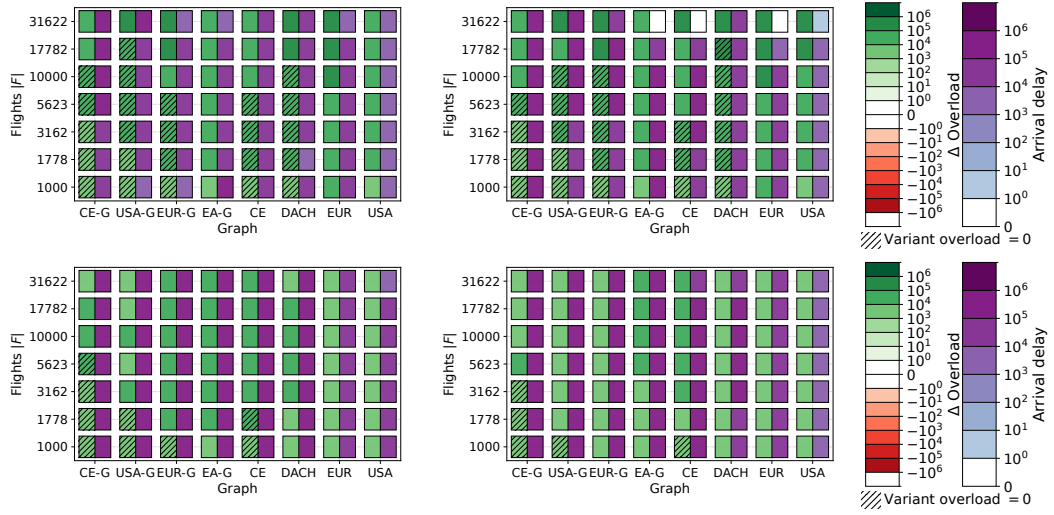

    \centering
    \begin{subfigure}[t]{0.4\textwidth}
    \includegraphics[width=0.99\textwidth]{imgs/03_scaling/20260603__variant_id-01_ASPaeroFlow__capacity-20.pdf}
    \end{subfigure}    
    \begin{subfigure}[t]{0.4\textwidth}
    \includegraphics[width=0.99\textwidth]{imgs/03_scaling/20260603__variant_id-0_Sequential__capacity-20.pdf}
    \end{subfigure}
    \begin{subfigure}[t]{0.18\textwidth}
    \includegraphics[width=0.99\textwidth]{imgs/03_scaling/20260603__variant_id-01_ASPaeroFlow__capacity-20__legend.pdf}
    \end{subfigure}

    \begin{subfigure}[t]{0.4\textwidth}
    \includegraphics[width=0.99\textwidth]{imgs/03_scaling/20260603__variant_id-03A_CASA__capacity-20.pdf}
    \end{subfigure}
    \begin{subfigure}[t]{0.4\textwidth}
    \includegraphics[width=0.99\textwidth]{imgs/03_scaling/20260603__variant_id-02_RerouteDelay__capacity-20.pdf}
    \end{subfigure}
    \begin{subfigure}[t]{0.18\textwidth}
    \includegraphics[width=0.99\textwidth]{imgs/03_scaling/20260603__variant_id-01_ASPaeroFlow__capacity-20__legend.pdf}
    \end{subfigure}
    \vspace{-0.7cm}
    \caption{
    $\Delta$-Scaling results for $10\%$ nominal capacity of $\text{ASPaeroFlow}_{r,d}$ (left-top), $\text{ASPaeroFlow}_{seq}$ (right-top), CASA (left-bottom), and $\text{ATFM}_{r,d}$ (right-bottom).
    The X-axis shows the graph size $|V|$, the Y-axis the number of flights $|F|$.
    For each $|F|,|V|$-pair, the rectangle shows the $\Delta$-overload reduction (left-half) and the arrival delay (right-half).
    }
    \label{fig:scaling-experiments}
    \vspace{-0.3cm}
\end{figure}

\begin{table}[t]
\centering
\resizebox{\textwidth}{!}{%
\begin{tabular}{lrrrrrr}
\toprule
Variant & Overload [\#] & Arrival Delay [h] & Sector-Number [\#] & Sector-Diff [\#] & Reroute [\#] & Reconfig [\#] \\
\midrule
$ASPaeroFlow_{seq}$ & $2777.62 \pm 379.70$ & $18818.67 \pm 1543.41$ & $87013.82 \pm 3392.09$ & $400.93 \pm 23.63$ & $749.34 \pm 42.30$ & $41206.49 \pm 3084.78$ \\
$ASPaeroFlow_{r,d}$ & $4073.11 \pm 547.13$ & $5955.45 \pm 628.65$ & $86314.66 \pm 3348.89$ & $337.65 \pm 17.51$ & $5187.05 \pm 171.84$ & $35137.84 \pm 2516.86$ \\
$ASPaeroFlow_{\neg r, d}$ & $5052.01 \pm 596.82$ & $8390.40 \pm 763.23$ & $89104.99 \pm 3441.68$ & $314.24 \pm 16.10$ & $5020.13 \pm 167.78$ & $35308.15 \pm 2407.23$ \\
$ASPaeroFlow_{r, d_p}$ & $5419.77 \pm 632.08$ & $2766.05 \pm 254.54$ & $82624.35 \pm 3152.47$ & $301.95 \pm 14.41$ & $5091.58 \pm 168.16$ & $28926.61 \pm 1888.04$ \\
$ASPaeroFlow_{\neg r, d_p}$ & $6532.10 \pm 669.41$ & $3659.26 \pm 298.22$ & $84864.11 \pm 3231.57$ & $265.40 \pm 12.27$ & $4842.18 \pm 163.09$ & $27296.11 \pm 1692.40$ \\
CASA & $12508.41 \pm 999.37$ & $72598.35 \pm 3113.78$ & $109210.61 \pm 4130.40$ & $0.00 \pm 0.00$ & $1838.58 \pm 67.88$ & $0.00 \pm 0.00$ \\
$ATFM_{r,d}$ & $13415.75 \pm 1017.33$ & $54205.15 \pm 2245.30$ & $99519.31 \pm 3736.04$ & $0.00 \pm 0.00$ & $1801.22 \pm 64.11$ & $0.00 \pm 0.00$ \\
$ATFM_{\neg r, d}$ & $13443.02 \pm 1019.90$ & $63289.14 \pm 2557.06$ & $103313.06 \pm 3888.24$ & $0.00 \pm 0.00$ & $1853.86 \pm 67.93$ & $0.00 \pm 0.00$ \\
$ATFM_{r,d_p}$ & $15876.31 \pm 1044.47$ & $18716.13 \pm 730.39$ & $95244.43 \pm 3557.80$ & $0.00 \pm 0.00$ & $923.25 \pm 25.81$ & $0.00 \pm 0.00$ \\
ASP & $16850.86 \pm 1047.58$ & $1187.87 \pm 89.26$ & $59385.92 \pm 2117.41$ & $302.13 \pm 23.19$ & $538.25 \pm 39.60$ & $2315.21 \pm 185.23$ \\
MIP & $17073.24 \pm 1046.26$ & $0.00 \pm 0.00$ & $59117.64 \pm 2121.67$ & $0.00 \pm 0.00$ & $9.80 \pm 1.79$ & $0.00 \pm 0.00$ \\
Initial & $17078.59 \pm 1046.21$ & $0.00 \pm 0.00$ & $59112.50 \pm 2121.75$ & $0.00 \pm 0.00$ & $0.00 \pm 0.00$ & $0.00 \pm 0.00$ \\
\bottomrule
\end{tabular}
}
\caption{
Results aggregated by variant over the large scenarios with SEM.
Besides overloads, we show all lexicographic soft constraint values.
}
\label{tab:summary-by-variant-new}
\end{table}
\begin{figure}
    \centering
    \includegraphics[width=10cm]{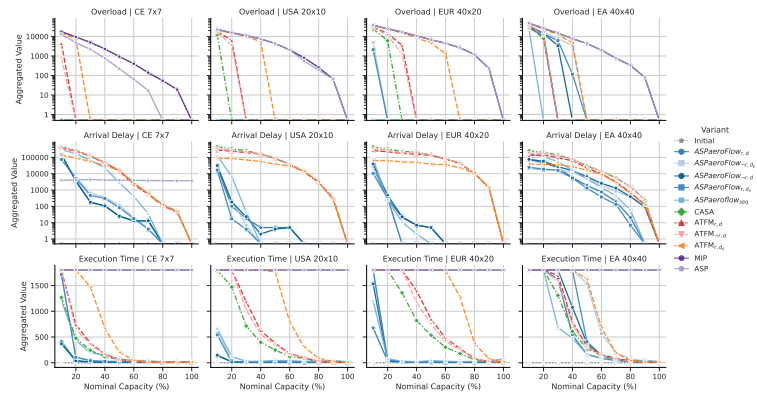}
    \caption{
    Nominal capacity scaling results on heuristic variants for grid graphs.
    }
    \label{fig:nominal-grid-graphs}
\end{figure}


\begin{table}[t]
\centering
\scriptsize
\resizebox{\textwidth}{!}{%
\begin{tabular}{lrrrrrrrr}
\toprule
Nominal Capacity & Overload [\#] & Avg. $\Delta$/step [\#] & Avg. $\Delta$/s [\#/s] & Arrival Delay [h] & Sector-Number [\#] & Sector-Diff [\#] & Reroute [\#] & Reconfig [\#] \\
\midrule
PCAP010 & $58664.8 \pm 1906.8$ & $28.73 \pm 4.47$ & $19.73 \pm 0.75$ & $76068.1 \pm 2861.7$ & $125976.6 \pm 4309.6$ & $468.0 \pm 24.8$ & $4421.2 \pm 147.2$ & $66323.2 \pm 3700.9$ \\
PCAP020 & $26061.7 \pm 944.4$ & $75.31 \pm 7.21$ & $81.11 \pm 4.09$ & $47543.7 \pm 2104.6$ & $113153.1 \pm 3946.0$ & $335.2 \pm 16.7$ & $3974.4 \pm 138.0$ & $35999.3 \pm 1978.1$ \\
PCAP030 & $12846.2 \pm 512.3$ & $196.72 \pm 17.33$ & $104.39 \pm 5.26$ & $33518.3 \pm 1613.7$ & $98342.0 \pm 3418.3$ & $239.1 \pm 12.0$ & $3558.3 \pm 128.7$ & $18128.1 \pm 971.9$ \\
PCAP040 & $6153.5 \pm 275.6$ & $225.25 \pm 19.29$ & $87.14 \pm 4.40$ & $22652.7 \pm 1215.1$ & $87035.7 \pm 2966.9$ & $168.3 \pm 8.5$ & $3136.5 \pm 120.9$ & $9541.5 \pm 456.9$ \\
PCAP050 & $2811.9 \pm 146.7$ & $197.36 \pm 16.33$ & $59.35 \pm 2.99$ & $13892.7 \pm 882.9$ & $80489.2 \pm 2727.7$ & $116.1 \pm 6.9$ & $2688.7 \pm 107.1$ & $5235.4 \pm 234.1$ \\
PCAP060 & $1236.2 \pm 75.0$ & $150.56 \pm 13.19$ & $38.88 \pm 1.96$ & $8407.2 \pm 636.7$ & $75189.7 \pm 2509.4$ & $85.2 \pm 7.9$ & $2070.5 \pm 89.8$ & $2899.7 \pm 124.4$ \\
PCAP070 & $465.1 \pm 33.5$ & $113.20 \pm 10.70$ & $22.12 \pm 1.08$ & $4375.9 \pm 379.2$ & $71292.1 \pm 2356.2$ & $62.5 \pm 7.5$ & $1505.7 \pm 67.4$ & $1652.7 \pm 77.5$ \\
PCAP080 & $143.6 \pm 10.9$ & $50.97 \pm 4.01$ & $9.77 \pm 0.41$ & $1182.4 \pm 94.6$ & $68711.7 \pm 2265.2$ & $53.0 \pm 7.5$ & $1118.5 \pm 55.0$ & $1123.7 \pm 62.2$ \\
PCAP090 & $34.5 \pm 2.5$ & $14.31 \pm 0.83$ & $3.31 \pm 0.11$ & $250.0 \pm 26.5$ & $66360.3 \pm 2188.3$ & $42.2 \pm 7.0$ & $693.5 \pm 39.0$ & $722.3 \pm 53.5$ \\
PCAP100 & $0.0 \pm 0.0$ & -- & -- & $97.6 \pm 24.5$ & $59137.5 \pm 1936.4$ & $32.2 \pm 7.8$ & $45.5 \pm 10.9$ & $199.6 \pm 51.8$ \\
\bottomrule
\end{tabular}
}
\caption{Results aggregated by nominal capacity}
\label{tab:summary-by-nominal-capacity}
\end{table}

\begin{table}[t]
\centering
\scriptsize
\resizebox{\textwidth}{!}{%
\begin{tabular}{lrrrrrrrr}
\toprule
Instance Size & Overload [\#] & Avg. $\Delta$/step [\#] & Avg. $\Delta$/s [\#/s] & Arrival Delay [h] & Sector-Number [\#] & Sector-Diff [\#] & Reroute [\#] & Reconfig [\#] \\
\midrule
1000 & $622.04 \pm 36.47$ & $11.28 \pm 0.59$ & $17.89 \pm 0.85$ & $9816.69 \pm 598.18$ & $84831.62 \pm 2520.56$ & $236.71 \pm 14.29$ & $334.91 \pm 6.86$ & $17018.73 \pm 1411.63$ \\
1778 & $1182.72 \pm 64.76$ & $25.61 \pm 2.72$ & $29.76 \pm 1.53$ & $13996.28 \pm 853.58$ & $85562.05 \pm 2537.53$ & $217.76 \pm 12.82$ & $563.99 \pm 11.97$ & $16998.16 \pm 1416.77$ \\
3162 & $2410.82 \pm 121.58$ & $38.84 \pm 2.59$ & $46.28 \pm 2.50$ & $19398.08 \pm 1144.41$ & $84990.31 \pm 2520.13$ & $155.27 \pm 9.55$ & $946.32 \pm 20.71$ & $15270.74 \pm 1280.41$ \\
5623 & $4923.73 \pm 229.35$ & $65.35 \pm 4.03$ & $66.23 \pm 3.69$ & $24974.66 \pm 1454.29$ & $84905.77 \pm 2524.25$ & $152.85 \pm 9.76$ & $1575.30 \pm 35.61$ & $15541.77 \pm 1302.17$ \\
10000 & $9662.92 \pm 412.66$ & $119.48 \pm 7.61$ & $70.24 \pm 3.75$ & $28644.44 \pm 1562.53$ & $84104.16 \pm 2485.46$ & $135.99 \pm 8.85$ & $2639.22 \pm 61.32$ & $13648.29 \pm 1163.46$ \\
17782 & $19239.11 \pm 758.16$ & $203.02 \pm 12.94$ & $58.41 \pm 3.00$ & $28231.74 \pm 1474.03$ & $83654.16 \pm 2472.67$ & $122.83 \pm 8.28$ & $4086.57 \pm 101.62$ & $12050.61 \pm 1031.62$ \\
31622 & $37850.79 \pm 1401.39$ & $354.95 \pm 22.75$ & $42.35 \pm 2.04$ & $20530.23 \pm 989.14$ & $83933.42 \pm 2476.91$ & $99.94 \pm 6.69$ & $6102.58 \pm 169.38$ & $8749.45 \pm 685.72$ \\
\bottomrule
\end{tabular}
}
\caption{Results aggregated by instance size (across all scenarios and options)}
\label{tab:summary-by-size}
\end{table}

\begin{table}[t]
\centering
\scriptsize
\resizebox{\textwidth}{!}{%
\begin{tabular}{lrrrrrrrr}
\toprule
Scenario & Overload [\#] & Avg. $\Delta$/step [\#] & Avg. $\Delta$/s [\#/s] & Arrival Delay [h] & Sector-Number [\#] & Sector-Diff [\#] & Reroute [\#] & Reconfig [\#] \\
\midrule
CE-7x7 & $3439.66 \pm 250.15$ & $42.13 \pm 4.39$ & $27.98 \pm 1.33$ & $38921.81 \pm 2165.26$ & $2129.42 \pm 25.39$ & $47.27 \pm 4.17$ & $3344.47 \pm 118.19$ & $850.52 \pm 26.31$ \\
USA-20x10 & $6058.08 \pm 358.06$ & $289.41 \pm 16.60$ & $105.03 \pm 4.52$ & $36909.71 \pm 1839.85$ & $2996.98 \pm 36.20$ & $78.91 \pm 8.46$ & $3167.77 \pm 114.51$ & $1712.30 \pm 61.52$ \\
EUR-40x20 & $9925.59 \pm 566.90$ & $466.93 \pm 25.39$ & $129.01 \pm 5.18$ & $36051.96 \pm 1679.41$ & $4921.12 \pm 57.56$ & $179.10 \pm 13.08$ & $2992.11 \pm 108.44$ & $8164.45 \pm 343.52$ \\
EA-40x40 & $14532.95 \pm 773.86$ & $9.05 \pm 0.87$ & $7.27 \pm 0.16$ & $23243.94 \pm 978.59$ & $10452.80 \pm 92.42$ & $30.37 \pm 0.99$ & $2672.54 \pm 93.02$ & $3522.50 \pm 160.06$ \\
CE & $11522.71 \pm 720.33$ & $13.97 \pm 1.09$ & $15.66 \pm 0.37$ & $11973.20 \pm 634.92$ & $8231.95 \pm 73.63$ & $36.70 \pm 1.33$ & $1671.47 \pm 57.46$ & $3207.36 \pm 160.67$ \\
DACH & $12101.47 \pm 757.77$ & $83.61 \pm 4.62$ & $76.77 \pm 3.01$ & $13136.04 \pm 644.07$ & $54483.17 \pm 458.80$ & $188.33 \pm 6.57$ & $2148.88 \pm 89.23$ & $12238.02 \pm 589.19$ \\
EUR & $17593.56 \pm 1113.67$ & $12.21 \pm 0.64$ & $11.14 \pm 0.46$ & $4269.18 \pm 165.51$ & $293145.55 \pm 2114.42$ & $457.11 \pm 19.83$ & $1714.96 \pm 86.80$ & $50161.32 \pm 2598.47$ \\
USA & $11559.82 \pm 796.91$ & $18.16 \pm 0.91$ & $5.61 \pm 0.19$ & $1885.14 \pm 72.23$ & $300189.29 \pm 2063.68$ & $263.76 \pm 14.93$ & $857.97 \pm 46.81$ & $33603.80 \pm 2280.28$ \\
\bottomrule
\end{tabular}
}
\caption{Results aggregated by scenario (across all instances and options)}
\label{tab:summary-by-scenario}
\end{table}



  \begin{figure}[htb]
      \centering
      \begin{subfigure}[t]{0.45\textwidth}
          \includegraphics[height=5cm]{imgs/04_initial_sectors_appendix/gabriel_00-0-CENTRAL-EUROPE-7x7-2019-06-01--2019-06-30-CAP-ENROUTE-1200-CLUSTERSIZE-1.pdf}
        \end{subfigure}
        \begin{subfigure}[t]{0.45\textwidth}
          \includegraphics[height=5cm]{imgs/05_initial_overload_appendix/00-0-CENTRAL-EUROPE-7x7-2019-06-01--2019-06-30-CAP-ENROUTE-1200-CLUSTERSIZE-1.pdf}
        \end{subfigure}
        \begin{subfigure}[t]{0.08\textwidth}
          \includegraphics[height=5cm]{imgs/05_initial_overload_appendix/00-0-CENTRAL-EUROPE-7x7-2019-06-01--2019-06-30-CAP-ENROUTE-1200-CLUSTERSIZE-1_legends.pdf}
        \end{subfigure}

        \begin{subfigure}[t]{0.4\textwidth}
            \includegraphics[width=5cm]{imgs/06_flight_counts/20260617_00.pdf}            
        \end{subfigure}
        \begin{subfigure}[t]{0.59\textwidth}
            \includegraphics[width=8cm]{imgs/07_airport_counts/00-0-CENTRAL-EUROPE-7x7-2019-06-01--2019-06-30-CAP-ENROUTE-1200-CLUSTERSIZE-1__DATA_S31622_42__departure_airports.pdf} 
        \end{subfigure}
        
        \caption{
        Left top: Initial navpoint-sector figure.
        Right top: Initial overload and trajectory figure.
        Left bottom: Filed flights departure times.
        Right bottom: Histogram of departure airports.
        Instance: 31622 flights, CE-G.
        }
      \label{fig:schematic-general-graph}
      \vspace{-0.5cm}
      \centering
      \begin{subfigure}[t]{0.45\textwidth}
          \includegraphics[height=5cm]{imgs/04_initial_sectors_appendix/gabriel_01-0-USA-EAST-COAST-20x10-2019-06-01--2019-06-30-CAP-ENROUTE-1200-CLUSTERSIZE-2.pdf}
        \end{subfigure}
        \begin{subfigure}[t]{0.45\textwidth}
          \includegraphics[height=5cm]{imgs/05_initial_overload_appendix/01-0-USA-EAST-COAST-20x10-2019-06-01--2019-06-30-CAP-ENROUTE-1200-CLUSTERSIZE-2.pdf}
        \end{subfigure}
        \begin{subfigure}[t]{0.08\textwidth}
          \includegraphics[height=5cm]{imgs/05_initial_overload_appendix/01-0-USA-EAST-COAST-20x10-2019-06-01--2019-06-30-CAP-ENROUTE-1200-CLUSTERSIZE-2_legends.pdf}
        \end{subfigure}

        \begin{subfigure}[t]{0.4\textwidth}
            \includegraphics[width=5cm]{imgs/06_flight_counts/20260617_01.pdf}            
        \end{subfigure}
        \begin{subfigure}[t]{0.59\textwidth}
            \includegraphics[width=8cm]{imgs/07_airport_counts/01-0-USA-EAST-COAST-20x10-2019-06-01--2019-06-30-CAP-ENROUTE-1200-CLUSTERSIZE-2__DATA_S31622_42__departure_airports.pdf} 
        \end{subfigure}
        
        \caption{
        Left top: Initial navpoint-sector figure.
        Right top: Initial overload and trajectory figure.
        Left bottom: Filed flights departure times.
        Right bottom: Histogram of departure airports.
        Instance: 31622 flights, USA-G.
        }
      \label{fig:schematic-general-graph}
      \vspace{-0.5cm}
  \end{figure}

  \begin{figure}[t]
      \centering
      \begin{subfigure}[t]{0.45\textwidth}
          \includegraphics[height=5cm]{imgs/04_initial_sectors_appendix/gabriel_02-0-MAJOR-EUROPE-40x20-2019-06-01--2019-06-30-CAP-ENROUTE-1200-CLUSTERSIZE-8.pdf}
        \end{subfigure}
        \begin{subfigure}[t]{0.45\textwidth}
          \includegraphics[height=5cm]{imgs/05_initial_overload_appendix/02-0-MAJOR-EUROPE-40x20-2019-06-01--2019-06-30-CAP-ENROUTE-1200-CLUSTERSIZE-8.pdf}
        \end{subfigure}
        \begin{subfigure}[t]{0.08\textwidth}
          \includegraphics[height=5cm]{imgs/05_initial_overload_appendix/02-0-MAJOR-EUROPE-40x20-2019-06-01--2019-06-30-CAP-ENROUTE-1200-CLUSTERSIZE-8_legends.pdf}
        \end{subfigure}

        \begin{subfigure}[t]{0.4\textwidth}
            \includegraphics[width=5cm]{imgs/06_flight_counts/20260617_02.pdf}            
        \end{subfigure}
        \begin{subfigure}[t]{0.59\textwidth}
            \includegraphics[width=8cm]{imgs/07_airport_counts/02-0-MAJOR-EUROPE-40x20-2019-06-01--2019-06-30-CAP-ENROUTE-1200-CLUSTERSIZE-8__DATA_S31622_42__departure_airports.pdf} 
        \end{subfigure}
        
        \caption{
        Left top: Initial navpoint-sector figure.
        Right top: Initial overload and trajectory figure.
        Left bottom: Filed flights departure times.
        Right bottom: Histogram of departure airports.
        Instance: 31622 flights, EUR-G.
        }
      \label{fig:schematic-general-graph}
      \vspace{-0.5cm}
  \end{figure}

  \begin{figure}[t]
      \centering
      \begin{subfigure}[t]{0.45\textwidth}
          \includegraphics[height=5cm]{imgs/04_initial_sectors_appendix/gabriel_03-0-EAST-ASIA-40x40-2019-06-01--2019-06-30-CAP-ENROUTE-1200-CLUSTERSIZE-16.pdf}
        \end{subfigure}
        \begin{subfigure}[t]{0.45\textwidth}
          \includegraphics[height=5cm]{imgs/05_initial_overload_appendix/03-0-EAST-ASIA-40x40-2019-06-01--2019-06-30-CAP-ENROUTE-1200-CLUSTERSIZE-16.pdf}
        \end{subfigure}
        \begin{subfigure}[t]{0.08\textwidth}
          \includegraphics[height=5cm]{imgs/05_initial_overload_appendix/03-0-EAST-ASIA-40x40-2019-06-01--2019-06-30-CAP-ENROUTE-1200-CLUSTERSIZE-16_legends.pdf}
        \end{subfigure}

        \begin{subfigure}[t]{0.4\textwidth}
            \includegraphics[width=5cm]{imgs/06_flight_counts/20260617_03.pdf}            
        \end{subfigure}
        \begin{subfigure}[t]{0.59\textwidth}
            \includegraphics[width=8cm]{imgs/07_airport_counts/03-0-EAST-ASIA-40x40-2019-06-01--2019-06-30-CAP-ENROUTE-1200-CLUSTERSIZE-16__DATA_S31622_42__departure_airports.pdf} 
        \end{subfigure}
        
        \caption{
        Left top: Initial navpoint-sector figure.
        Right top: Initial overload and trajectory figure.
        Left bottom: Filed flights departure times.
        Right bottom: Histogram of departure airports.
        Instance: 31622 flights, EA-G.
        }
      \label{fig:schematic-general-graph}
      \vspace{-0.5cm}
  \end{figure}

  \begin{figure}[t]
      \centering
      \begin{subfigure}[t]{0.45\textwidth}
          \includegraphics[height=5cm]{imgs/04_initial_sectors_appendix/gabriel_04-0-DACH-2019-06-01--2019-06-30-CAP-ENROUTE-1200-CLUSTERSIZE-50-GABRIEL-GRAPH.pdf}
        \end{subfigure}
        \begin{subfigure}[t]{0.45\textwidth}
          \includegraphics[height=5cm]{imgs/05_initial_overload_appendix/04-0-DACH-2019-06-01--2019-06-30-CAP-ENROUTE-1200-CLUSTERSIZE-50-GABRIEL-GRAPH.pdf}
        \end{subfigure}
        \begin{subfigure}[t]{0.08\textwidth}
          \includegraphics[height=5cm]{imgs/05_initial_overload_appendix/04-0-DACH-2019-06-01--2019-06-30-CAP-ENROUTE-1200-CLUSTERSIZE-50-GABRIEL-GRAPH_legends.pdf}
        \end{subfigure}

        \begin{subfigure}[t]{0.4\textwidth}
            \includegraphics[width=5cm]{imgs/06_flight_counts/20260617_04.pdf}            
        \end{subfigure}
        \begin{subfigure}[t]{0.59\textwidth}
            \includegraphics[width=6cm]{imgs/07_airport_counts/04-0-DACH-2019-06-01--2019-06-30-CAP-ENROUTE-1200-CLUSTERSIZE-50-GABRIEL-GRAPH__DATA_S31622_42__departure_airports.pdf} 
        \end{subfigure}
        
        \caption{
        Left top: Initial navpoint-sector figure.
        Right top: Initial overload and trajectory figure.
        Left bottom: Filed flights departure times.
        Right bottom: Histogram of departure airports.
        Instance: 31622 flights, DACH.
        }
      \label{fig:schematic-general-graph}
      \vspace{-0.5cm}
  \end{figure}

  \begin{figure}[t]
      \centering
      \begin{subfigure}[t]{0.45\textwidth}
          \includegraphics[height=5cm]{imgs/04_initial_sectors_appendix/gabriel_05-0-EUROPE-2019-06-01--2019-06-30-CAP-ENROUTE-1200-CLUSTERSIZE-200-GABRIEL-GRAPH_compressed.pdf}
        \end{subfigure}
        \begin{subfigure}[t]{0.45\textwidth}
          \includegraphics[height=5cm]{imgs/05_initial_overload_appendix/05-0-EUROPE-2019-06-01--2019-06-30-CAP-ENROUTE-1200-CLUSTERSIZE-200-GABRIEL-GRAPH_compressed.pdf}
        \end{subfigure}
        \begin{subfigure}[t]{0.08\textwidth}
          \includegraphics[height=5cm]{imgs/05_initial_overload_appendix/05-0-EUROPE-2019-06-01--2019-06-30-CAP-ENROUTE-1200-CLUSTERSIZE-200-GABRIEL-GRAPH_legends.pdf}
        \end{subfigure}

        \begin{subfigure}[t]{0.4\textwidth}
            \includegraphics[width=5cm]{imgs/06_flight_counts/20260617_05.pdf}            
        \end{subfigure}
        \begin{subfigure}[t]{0.59\textwidth}
            \includegraphics[width=6cm]{imgs/07_airport_counts/05-0-EUROPE-2019-06-01--2019-06-30-CAP-ENROUTE-1200-CLUSTERSIZE-200-GABRIEL-GRAPH__DATA_S31622_42__departure_airports.pdf} 
        \end{subfigure}
        
        \caption{
        Left top: Initial navpoint-sector figure.
        Right top: Initial overload and trajectory figure.
        Left bottom: Filed flights departure times.
        Right bottom: Histogram of departure airports.
        Instance: 31622 flights, EUR.
        }
      \label{fig:schematic-general-graph}
      \vspace{-0.5cm}
  \end{figure}

  \begin{figure}[t]
      \centering
      \begin{subfigure}[t]{0.45\textwidth}
          \includegraphics[width=6cm]{imgs/04_initial_sectors_appendix/gabriel_06-0-USA-MAINLAND-2019-06-01--2019-06-30-CAP-ENROUTE-1200-CLUSTERSIZE-800-GABRIEL-GRAPH_compressed.pdf}
        \end{subfigure}
        \begin{subfigure}[t]{0.45\textwidth}
          \includegraphics[width=6cm]{imgs/05_initial_overload_appendix/06-0-USA-MAINLAND-2019-06-01--2019-06-30-CAP-ENROUTE-1200-CLUSTERSIZE-800-GABRIEL-GRAPH_compressed.pdf}
        \end{subfigure}
        \begin{subfigure}[t]{0.08\textwidth}
          \includegraphics[height=3cm]{imgs/05_initial_overload_appendix/06-0-USA-MAINLAND-2019-06-01--2019-06-30-CAP-ENROUTE-1200-CLUSTERSIZE-800-GABRIEL-GRAPH_legends.pdf}
        \end{subfigure}

        \begin{subfigure}[t]{0.4\textwidth}
            \includegraphics[width=5cm]{imgs/06_flight_counts/20260617_06.pdf}            
        \end{subfigure}
        \begin{subfigure}[t]{0.59\textwidth}
            \includegraphics[width=6cm]{imgs/07_airport_counts/06-0-USA-MAINLAND-2019-06-01--2019-06-30-CAP-ENROUTE-1200-CLUSTERSIZE-800-GABRIEL-GRAPH__DATA_S31622_42__departure_airports.pdf} 
        \end{subfigure}
        
        \caption{
        Left top: Initial navpoint-sector figure.
        Right top: Initial overload and trajectory figure.
        Left bottom: Filed flights departure times.
        Right bottom: Histogram of departure airports.
        Instance: 31622 flights, USA.
        }
      \label{fig:schematic-general-graph}
      \vspace{-0.5cm}
  \end{figure}
  


%
\end{document}